\documentclass{article}

\usepackage[final,main]{neurips_2026}


\usepackage[utf8]{inputenc} 
\usepackage[T1]{fontenc}    
\usepackage{hyperref}       
\usepackage{url}            
\usepackage{booktabs}       
\usepackage{amsfonts}       
\usepackage{nicefrac}       
\usepackage{microtype}      
\usepackage{xcolor}         
\usepackage{graphicx}
\usepackage{amsmath}
\usepackage{amssymb}
\usepackage{amsthm}
\usepackage{tabularx}
\usepackage{enumitem}
\usepackage{makecell}
\usepackage{algorithm}
\usepackage{algorithmic}

\usepackage{comment}

\newcommand{\E}{\mathbb{E}}

\newif\ifshowchanges
\showchangestrue
\ifshowchanges
  \DeclareRobustCommand{\edit}[1]{{\color{red}#1}}
  \DeclareRobustCommand{\edit}[1]{{#1}}
  \DeclareRobustCommand{\del}[1]{{}}
  \newenvironment{editblock}{\par}{\par}
  \newenvironment{movedblock}{\par}{\par}
\else
  \DeclareRobustCommand{\edit}[1]{#1}
  \DeclareRobustCommand{\del}[1]{}
  \newenvironment{editblock}{\par}{\par}
  \excludecomment{delblock}
  \newenvironment{movedblock}{\par}{\par}
\fi

\newtheorem{thm}{Theorem}
\newtheorem{lem}[thm]{Lemma}

\newtheorem{prop}{Proposition}[section]

\newtheorem{defn}{Definition}

\newtheorem{rem}{Remark}[section]

\makeatletter
\newcommand{\symfootnotetext}[3]{%
  \begingroup
    \renewcommand{\thefootnote}{#2}%
    \def\Hy@footnote@currentHref{Hfootnote.tab#1}%
    \footnotetext{#3}%
  \endgroup}
\makeatother

\title{Statistical Complexity of \\Soft Bellman Residual Minimization}

\author{%
  Enoch Hyunwook Kang \\
  Foster School of Business \\
  University of Washington \\
  \texttt{ehwkang@uw.edu} \\
  \And
  Kyoungseok Jang \\
  Department of AI \\
  Chung Ang University \\
  \texttt{ksjang@cau.ac.kr} \\
}

\begin{document}

\maketitle

\begin{abstract}
  Offline reinforcement learning and offline inverse reinforcement learning aim to recover near-optimal value functions or reward models from a fixed batch of logged trajectories, yet current practice still struggles to enforce Bellman consistency scalably. Bellman residual minimization (BRM) has emerged as an attractive remedy. While globally convergent gradient-based soft (i.e., entropy-regularized) BRM methods for neural networks have recently been established, their statistical complexity under stochastic gradient descent remains largely unknown. In this paper, we address this theoretical gap for soft BRM. Through a Lyapunov-based analysis under a uniform PL condition in the auxiliary variable, we establish an $\mathcal{O}(1 / n)$ primal argument stability bound up to optimization error and an $\mathcal{O}(1 / n)$ statistical complexity for the soft BRM objective under the stated sampling conditions.
\end{abstract}

\section{Introduction}

Reinforcement learning (RL) provides a principled framework for sequential decision-making, formalizing dynamic problems as Markov decision processes (MDPs) and learning policies that optimize long-horizon rewards. In many high-stakes domains, however, such as adaptive clinical trials \citep{zhao2009reinforcement}, recommendation systems \citep{afsar2022reinforcement}, and online education \citep{bassen2020reinforcement}, online interaction is unethical, dangerous, or prohibitively expensive \citep{jiang2024offline}. This motivates offline RL, where the learner has access only to a fixed batch of data collected by a behavior policy, and its sister field, inverse RL (IRL), which infers reward functions from logged expert trajectories.

Beyond the lack of online interaction, these problems pose both statistical and computational challenges: realistic state spaces are often high-dimensional or continuous, making exact dynamic programming intractable. Following \citet{jiang2024offline}, existing offline approaches can be broadly organized into three paradigms: approximate dynamic programming, such as fitted $Q$-iteration \citep{ernst2005tree}; marginalized importance sampling \citep{liu2018breaking}; and Bellman Residual Minimization (BRM) \citep{baird1995residual}. In this work, we adopt the BRM as our primary framework.

BRM reformulates dynamic programming as fitting a value function to the Bellman equation and minimizes the mean-squared Bellman discrepancy under the offline data distribution. Despite its conceptual appeal and empirical effectiveness, BRM has historically had limited theory for global optimality and convergence. A notable exception is SBEED \citep{dai2018sbeed}, which establishes stable convergence guarantees for the soft, i.e., entropy-regularized, BRM objective, replacing the Bellman max operator with a softmax\footnote{A more comprehensive discussion of related work is deferred to Appendix~\ref{appssec:related works}.}.

Building on this direction, \citet{kang2025empiricalriskminimizationapproach} analyzed the optimization landscape of soft BRM. After a classical bi-conjugate transformation, they showed that, for common $Q$-function classes such as neural networks\footnote{Following \citet{liu2022loss}, this result uses neural networks whose width scales with $\mathrm{radius}^{\mathrm{depth}}$.}, the mean-squared Bellman error (MSBE) admits a Polyak--{\L}ojasiewicz (PL) structure\footnote{The PL condition requires
$\tfrac12\|\nabla f(x)\|_2^2 \ge \mu(f(x)-f^\ast)$ for some $\mu>0$, and yields convergence guarantees comparable to strong convexity even without convexity. Here we impose PL on the outer minimization and the inner maximization, retaining the primal quadratic-growth condition.}. Under suitable learning rates and the stated geometry, stochastic gradient descent--ascent (SGDA) enjoys global convergence \citep{yang2020globalconvergencevariancereducedoptimization}. While this line of work substantially clarifies the optimization behavior of soft BRM, it leaves open the question of how statistical errors arising from finite offline data affect the learned policy\edit{.}\del{:}\footnote{The main statistical guarantee of this paper was later referenced and used in the recent version of \citet{kang2025empiricalriskminimizationapproach}; see Appendix~\ref{appssec:discussion} for details.} This raises the central question:
\begin{center}
    \textit{How many offline samples \edit{do} we need for BRM to get a near-optimal value function, under SGDA?}
\end{center}

\paragraph{Contributions.}
In this paper, we close this statistical gap for function classes satisfying the conditions below, including neural-network and linear parameterizations when their geometry is verified. Building on \edit{the bi-conjugate reformulation \citep{antos2008learning,dai2018sbeed} and} the PL structure identified by \citet{kang2025empiricalriskminimizationapproach}, we develop a Lyapunov potential tailored to PL optimization and combine on-average primal argument stability with a direct statistical analysis. Our main theorem shows that, for $n$ independent observations, SGDA attains an $\mathcal{O}(1/n+1/T)$ excess primal risk after $T$ iterations, with an MSBE interpretation under auxiliary-function realizability. Stationary dependent data enter through an explicit covariance factor. This complements the convex--concave minimax guarantees with Markov-sampled update indices in \citet{Wang2022MCSGM}\edit{; Table~\ref{tab:related-compact} compares assumptions and rates, and Appendix~\ref{appssec:discussion} explains why convex--concave stability arguments do not apply to soft BRM}.%
\symfootnotetext{901}{\ensuremath{\dagger}}{Table~\ref{tab:related-compact}: squared-loss bound behind the headline value-gap / regret guarantee.}%
\symfootnotetext{902}{\ensuremath{\ddagger}}{Table~\ref{tab:related-compact}: $O(n^{-1})$ via local Rademacher complexity for finite or linear classes.}
In particular:

\begin{enumerate}[leftmargin=2.1em]
\item We prove an $\mathcal{O}(1/n)$ on-average primal argument-stability bound for SGDA under inner PL, up to optimization error, without requiring independent observations for this stability bound.
\item Using the same PL geometry, we derive the $\mathcal{O}(1/n)$ generalization guarantee for BRM under the stated sampling conditions.
\item Our analysis is constructive, requires no variance reduction or extra regularization, and applies to neural-network parametrizations satisfying the stated PL, QG, and domain conditions.
\end{enumerate}

\paragraph{Technical Sketch.} We couple two SGDA runs on neighboring datasets using the same initialization and the same minibatch index sequence. The Lyapunov potential contracts in expectation by a factor $(1-c\eta_t)$, while stochastic perturbations enter through a term of order $\eta_t^2$. Comparing each run with its empirical primal minimizer gives the desired $O(1/n)$ stability term. Evaluating empirical gradients at a fixed population optimum then gives the statistical bound.

\begin{table}[t]
\centering
\small
\setlength{\tabcolsep}{4pt}
\caption{\edit{Statistical rates for prior works. BRM rates concern squared Bellman error or the corresponding population objective, subject to the qualifications below. The row of \citet{Wang2022MCSGM} concerns general excess primal population risk. \textbf{Oracle} means the analysis assumes the exact empirical (ERM or minimax) solution is available. \textbf{SGDA} means it bounds the output of stochastic gradient descent--ascent.}}
\label{tab:related-compact}
\begin{tabular}{@{}lllll@{}}
\toprule
Work & Hypothesis Class & Estimator & Key assumptions & Rate \\
\midrule
\citet{maillard2010finite}  & linear            & Oracle & generative model              & $O(n^{-1/2})$ \\
\citet{dai2018sbeed}        & general         & Oracle & $\beta$-mixing path          & $\tilde O(n^{-1/2})$ \\
\citet{chen2019information} & finite            & Oracle & realizability, completeness  & $\tilde O(n^{-1})$$^{\dagger}$ \\
\citet{duan2021risk}        & general           & Oracle & completeness                 & $O(n^{-1/2})$$^{\ddagger}$ \\
\citet{hu2025fast}          & linear            & Oracle & completeness, feature cov.   & $\tilde O(n^{-1})$$^{\dagger}$ \\
\citet{Wang2022MCSGM}   & general   & SGDA   & convex--strongly concave                    & $\tilde O(n^{-1/2})$ \\
\midrule
\textbf{Ours}               & general   & SGDA   & PL--PL, QG            & $O(n^{-1})$ \\
\bottomrule
\end{tabular}
\end{table}

\section{Setup and Backgrounds}\label{sec:setup}
\textbf{Markov Decision Process.}\;  Throughout, we focus on a single-agent decision-making problem interacting with a discounted Markov Decision Process (MDP) described by the tuple
$\bigl(\mathcal S,\mathcal A,P,r,\beta,\nu_0\bigr)$. A state is an element of the measurable space $\mathcal S$ and the agent chooses actions from the finite set $\mathcal A$.
For any state--action pair $(s,a)$ the transition kernel
$P(\cdot\,|\,s,a)$ gives a probability distribution over the next state and the reward function $r:\mathcal S\times\mathcal A\!\to\!\mathbb R$ gives the immediate reward. We study the entropy-regularized control problem with regularization coefficient $\eta>0$, where the Shannon entropy of a distribution $q\in\Delta_{\mathcal A}$ is
$
\mathcal{H}(q):=-\sum_{a\in\mathcal A}q(a)\log q(a).
$
The scalar $\beta\in(0,1)$ exponentially discounts rewards that occur further in the future and $\nu_0$ denotes the distribution of the starting state $s_0$. Throughout, the unsubscripted symbol $\eta$ denotes the entropy-regularization coefficient; later, the subscripted sequence $(\eta_t)_{t\ge 0}$ continues to denote the SGDA stepsizes.

\textbf{Policy and value functions.} \; A (stationary Markov) policy $\pi\in \Delta_{\mathcal{A}}^{\mathcal{S}}$ assigns every state $s$ to $\pi(\cdot\,|\,s)$, a distribution over actions $\mathcal{A}$.
Combined with the initial draw $s_0\sim\nu_0$, a policy induces a probability measure $\mathbb P_{\nu_0,\pi}$ on infinite trajectories $(s_0,a_0,s_1,a_1,\dots)$, and the corresponding expectation operator is written $\mathbb E_{\nu_0,\pi}$.
Under this setup, we consider the entropy-regularized optimal policy and its corresponding value functions defined as
\begin{align}
\pi^* & :=\operatorname{argmax}_{\pi \in \Delta_{\mathcal{A}}^{\mathcal{S}}} \mathbb{E}_{\nu_0, \pi}\left[\sum_{h=0}^{\infty} \beta^h \left(r(s_h,a_h)+\eta \mathcal{H}(\pi(\cdot\mid s_h))\right)\right]  \notag \allowdisplaybreaks\\
V^*(s) & :=\max _{\pi \in \Delta_{\mathcal{A}}^{\mathcal{S}}} \mathbb{E}_{\nu_0, \pi}\left[\sum_{h=0}^{\infty} \beta^h\left(r(s_h,a_h)+\eta \mathcal{H}(\pi(\cdot\mid s_h))\right) \mid s_0=s\right] \notag
\allowdisplaybreaks\\
Q^*(s, a) & :=\max _{\pi \in \Delta_{\mathcal{A}}^{\mathcal{S}}} \mathbb{E}_{\nu_0, \pi}\left[r(s_0,a_0)+\sum_{h=1}^{\infty} \beta^h\left(r(s_h,a_h)+\eta \mathcal{H}(\pi(\cdot\mid s_h))\right) \mid s_0=s, a_0=a\right]. \notag
\end{align}
One can show that the optimal policy $\pi^\ast$ and the value functions satisfy the following entropy-regularized Bellman optimality equations \citep[Section~3.1 and Appendix~B]{kang2025empiricalriskminimizationapproach}:
$$
\begin{aligned}
V^*(s) & =\max_{q\in \Delta_{\mathcal A}}\left\{\mathbb E_{a\sim q}\bigl[Q^*(s,a)\bigr]+\eta \mathcal H(q)\right\}=\eta \ln \left[\sum_{a \in \mathcal{A}} \exp \left(Q^*(s, a)/\eta\right)\right] \\
\pi^*(a \mid s) & =\frac{\exp \left(Q^*(s, a)/\eta\right)}{\sum_{a' \in \mathcal{A}} \exp \left(Q^*(s, a')/\eta\right)} = \exp \left(\frac{Q^*(s, a)-V^*(s)}{\eta}\right) \text { for } a \in \mathcal{A} \\
Q^*(s, a) & =r(s, a)+\beta \cdot \mathbb{E}_{s^{\prime} \sim P(\cdot\mid s, a)}\left[V^*(s') \mid s, a\right]
\end{aligned}
$$

\subsection{Bellman Residual Minimization}

\textbf{Bellman Error (Bellman Residual) and Temporal Difference Error.} \;
{Let $\mathcal Q$ be the set of bounded functions on $\mathcal S\times\mathcal A$, which contains $Q^*$ when rewards are bounded.} We introduce the \textit{(soft) Bellman optimality operator}, $\mathcal{T}:\mathcal{Q} \to \mathcal{Q}$, defined for any $Q \in \mathcal{Q}$ by:
$$
(\mathcal{T}Q)(s, a) := r(s, a)+\beta \cdot \mathbb{E}_{s^{\prime} \sim P(\cdot|s, a)}\left[V_Q(s') \mid s, a\right]
\;\text{where}\;
V_Q(s):=\eta \ln \left[\sum_{a \in \mathcal{A}} \exp \left(Q(s, a)/\eta\right)\right].
$$
Note that $Q^*$ is the unique fixed point of $\mathcal{T}$ \citep{rust1994structural}; $Q^*$ is the unique solution to the \emph{Bellman optimality equation}: $
\mathcal{T}Q^* = Q^*.
$
This uniqueness motivates the definition of the \emph{Bellman error} (Bellman residual) of $Q$ at $(s,a)$, $(\mathcal TQ)(s,a)-Q(s,a)$.
Computing $\mathcal{T}Q$ requires knowledge of the transition kernel $P$, which is typically unavailable in reinforcement learning. Instead, given a sampled transition $(s,a,s')$, we work with the \textit{sampled Bellman operator}, $\hat{\mathcal{T}}$, and its corresponding \textit{temporal-difference (TD) error}, $\delta_Q$. Given a single transition tuple $(s, a, s')$, they are defined as:
$$
\hat{\mathcal{T}}Q(s, a, s') := r(s, a) + \beta V_Q (s'), \qquad \delta_Q (s,a,s'):=\hat{\mathcal{T}}Q (s,a,s')-Q(s,a)
$$ Note that $\delta_Q$ is an unbiased, single-sample estimate of the Bellman error: $\mathbb E_{s'\sim P(\cdot\mid s,a)}[\delta_Q(s,a,s')]=(\mathcal TQ)(s,a)-Q(s,a)$.

\textbf{Bellman Residual Minimization.} \; {Note that both Bellman error (Bellman residual) and its proxy, the TD error, are functions of $(s, a)$. To find $Q$ that minimizes the Bellman error for all $(s,a)$, we can instead find $Q$ that minimizes expected square error on the offline data distribution. That is, we define} the \textit{Squared Bellman Error} at $(s,a)$ as
$\mathcal{L}_{\text{BE}}(Q)(s, a) := \left( (\mathcal{T}Q)(s, a) - Q(s, a) \right)^2
$ and minimize the \textit{Mean Squared Bellman Error} (MSBE), defined as:
$$
\overline{\mathcal{L}_{\text{BE}}}(Q) := \mathbb{E}_{(s,a)\sim \pi_D, \nu_0} \left[ \mathcal{L}_{\text{BE}}(Q)(s, a) \right]
$$
where $\pi_D$ is the policy used for collecting data. {The mean squared TD error (MSTDE), built from the squared TD error $\mathcal L_{\mathrm{TD}}(Q)(s,a,s'):=\delta_Q(s,a,s')^2$, is a biased proxy for MSBE because expectation and square do not commute. This is the} \textit{double sampling problem} \citep{antos2008learning}.
\edit{The bias equals $\beta^2\,\mathbb E_{s'}\bigl[(V_Q(s')-\mathbb E_{s'}V_Q(s'))^2\bigr]$ \citep[Appendix~C.1]{kang2025empiricalriskminimizationapproach} and cannot be removed without knowing $P$.} 

To resolve this issue, we employ the ``Bi-Conjugate Trick'' \citep{antos2008learning, dai2018sbeed, patterson2022generalized}\edit{, $\mathcal L_{\mathrm{BE}}(Q)(s,a)=\max_{h\in\mathbb R}\{2h\,\mathbb E_{s'}[\delta_Q(s,a,s')]-h^2\}$, which after substituting $\zeta:=(h-r(s, a)+Q(s, a))/\beta$ \citep[Appendix~C.1]{kang2025empiricalriskminimizationapproach} reads}
\begin{align}
    \mathcal{L}_{B E}&(Q)(s, a)=\mathbb{E}_{s^{\prime} \sim P(\cdot|s, a)}\left[\mathcal{L}_{T D}(Q)\left(s, a, s^{\prime}\right)\right] -\beta^2 \min _{\zeta \in \mathbb{R}} \mathbb{E}_{s^{\prime} \sim P(\cdot|s, a)}\left[\left(V_Q\left(s^{\prime}\right)-\zeta\right)^2\right] \label{eq:BEminimax}
\end{align}
Now define the per-sample objective
\begin{align}
f(Q,\zeta ;s,a,s'):=\mathcal{L}_{\mathrm{TD}}(Q)\left(s, a, s^{\prime}\right)-  \beta^2\left(V_Q\left(s^{\prime}\right)-\zeta\right)^2.
    \label{eq:per-sample-obj}
\end{align}
If we slightly abuse notation such that $\mathbb{E}_{s,a,s'} [*]=\mathbb{E}_{s,a \sim \nu_0, \pi_D} \left[\mathbb{E}_{s'\sim P(\cdot|s,a)}[*]\right]$, then minimizing MSBE {can be written as a min--max problem over $(Q,\zeta)$:
\begin{align*}
    \min_{Q\in\mathcal{Q}} \mathbb{E}_{(s,a,s^{\prime})} \left[f(Q,\tilde\zeta(Q);s,a,s')\right] \quad \text{where}\;\; \widetilde{\zeta}(Q)\in \operatorname*{argmin}_{\zeta \in \mathbb{R}^{S \times A}} \mathbb{E}_{(s,a,s')}\left[\left(V_Q\left(s^{\prime}\right)-\zeta(s, a)\right)^2\right].
\end{align*}

By parametrizing $\zeta$ as $\zeta_{v}$ and $Q$ as $Q_{w}$ using function classes such as neural networks and relaxing the definition of the per-sample objective
\begin{align}
    f(w,v;s,a,s'):=f(Q_w, \zeta_v;s,a,s')
    =\mathcal{L}_{\mathrm{TD}}(Q_w)\left(s, a, s^{\prime}\right)-  \beta^2\left(V_{Q_w}\left(s^{\prime}\right)-\zeta_v\right)^2,\label{eq:relaxed-per-sample}
\end{align} the resulting parametrized problem becomes
\begin{align}
\min_{w \in \mathcal{W}} \mathbb{E}_{(s,a,s')}\left[f(w,\tilde{v};s,a,s')\right]
\quad\text{where } \tilde{v}(w) \in \operatorname*{argmin}_{v \in \mathcal{V}} \mathbb{E}_{(s,a,s')}\left[\left(V_{Q_{w}}\left(s^{\prime}\right)-\zeta_{v}(s, a)\right)^2\right]\label{eq:paraExpRM}.
\end{align}
The population objective in~\eqref{eq:paraExpRM} equals MSBE when the auxiliary class realizes
$m_w(s,a):=\E[V_{Q_w}(S')\mid s,a]$. In general it equals MSBE minus
$\beta^2\inf_v\E[(m_w(S,A)-\zeta_v(S,A))^2]$; a uniform upper bound
$\epsilon_{\mathrm{aux}}$ on this approximation error adds
$\beta^2\epsilon_{\mathrm{aux}}$ to the excess-MSBE guarantee.

Let's consider Equation \eqref{eq:paraExpRM} as the expected risk minimization problem. Define $\mathcal{D}_{\pi_D,\nu_0} = \{z_i\}_{i=1}^N$ as the offline data collected from following $\pi_D$ starting from $\nu_0$, where each $z_i=(s_i,a_i,s'_i)$ denotes a sample consisting of the current state, the action taken, and the resulting next state. Then, the corresponding empirical risk minimization problem can be written as
\begin{align}
&\min_{w \in \mathcal{W}} \frac{1}{N}\sum_{\substack{z\in \mathcal{D}_{\pi_D,\nu_0}}}f(w,\tilde{v};z), \text{ where }\;\tilde{v}(w) \in \operatorname*{argmin}_{v \in \mathcal{V}} \frac{1}{N} \sum_{\substack{z \in \mathcal{D}_{\pi_D,\nu_0}}}\left[\left(V_{Q_{w}}\left(s^{\prime}\right)-\zeta_{v}(s, a)\right)^2\right]. \label{eq:EmpRM}
\end{align}

A canonical way of solving the mini-max problem is to apply the \edit{Stochastic Gradient Descent--Ascent (SGDA)} algorithm \citep{yang2020globalconvergencevariancereducedoptimization}.  \citet{kang2025empiricalriskminimizationapproach} proved that both \eqref{eq:paraExpRM} and \eqref{eq:EmpRM} satisfy the Polyak-{\L}ojasiewicz condition within a large enough ball around the initialization point for a Neural Network parametrization of $Q$ with a sufficient network width (width scaling with $\text{radius}^\text{depth}$). The SGDA guarantee used here is stated below under explicit inner PL, primal QG, and domain conditions. (For the detailed discussion on neural networks, see Appendix \ref{appssec:Neural}.)

\subsection[Stochastic Gradient Descent--Ascent Algorithm (SGDA)]{\edit{Stochastic Gradient Descent--Ascent Algorithm (SGDA)}}\label{sec:sgda}

As discussed earlier, SGDA solves \eqref{eq:EmpRM} by a descent step in $w$ and an ascent step in $v$ using minibatch gradients.
Let $\mathcal D=\{z_i\}_{i=1}^{n}$ denote the dataset, where each $z_i=(s_i,a_i,s'_i)$ denotes a sample consisting of the current state, the action taken, and the resulting next state.  Fix a minibatch size $B\in\{1,\dots,n\}$.  At round $t$ we draw an index batch $I_t$ of size $B$ (a multiset for sampling with replacement and a subset of $[n]$ otherwise), either \emph{with} or \emph{without} replacement (our theory does not depend on this choice). With $f$ standing for the relaxed per--sample saddle objective introduced in~\eqref{eq:relaxed-per-sample}, the averaged stochastic gradients are $    g_t^w=\nabla_w F_{I_t}(w_t,v_t),\quad
    g_t^v=\nabla_v F_{I_t}(w_t,v_t), \text{ where } \quad
    F_I(w,v):=\frac{1}{|I|}\sum_{j\in I}f(w,v;z_j)$ for any minibatch $I$. By a slight abuse of notation, when the subscript corresponds to the full dataset $\mathcal{D}$, we write $
F_{\mathcal{D}}(w,v) := \frac{1}{n} \sum_{z \in \mathcal{D}} f(w,v;z).
$

Conditionally on the dataset and the past iterates, the minibatch is drawn
uniformly and independently of the past. Thus
$\E[g_t^w\mid\mathcal F_t]=\nabla_wF_{\mathcal D}(w_t,v_t)$ and likewise
for $v$. Write $\sigma_B^2$ for a common conditional variance bound for
the two gradient blocks; $\sigma_B^2\le G^2/B$, and $\sigma_B^2=0$ for
an exact full-batch update. Set
$\kappa_v:=1+(L/\rho)^2$ and use primal and dual stepsizes
$\eta_t$ and $\kappa_v\eta_t$, respectively. \edit{Here $L$, $\rho$ and $G$ are the smoothness, inner-PL and gradient-bound constants of Conditions~\ref{asm:smooth}, \ref{asm:SC} and~\ref{asm:grad} in Section~\ref{sec:stability}.} Algorithm~\ref{alg:SGDA} is stated for a generic per-sample saddle objective $f$: with $f$ from~\eqref{eq:relaxed-per-sample} it solves the BRM problem~\eqref{eq:EmpRM}, and Section~\ref{sec:stability} analyzes it for any $f$ satisfying Conditions~\ref{asm:smooth}--\ref{asm:unique}.
\begin{algorithm}[H]
  \caption{Minibatch SGDA on an empirical saddle objective $F_{\mathcal D}$}
  \label{alg:SGDA}
  \begin{algorithmic}
    \STATE {\bfseries Input:} Dataset $\mathcal D=\{z_i\}_{i=1}^{n}$, per-sample objective $f$, minibatch size $B$, stepsizes $(\eta_t)$, dual ratio $\kappa_v$, initial $(w_0,v_0)$
    \FOR{$t=0$ {\bfseries to} $T-1$}
    \STATE Draw a fresh uniform index batch $I_t$ of size $B$
    \STATE $g_t^w\leftarrow\frac1B\sum_{i\in I_t}\nabla_w f(w_t,v_t;z_i)$
    \STATE $g_t^v\leftarrow\frac1B\sum_{i\in I_t}\nabla_v f(w_t,v_t;z_i)$
    \STATE $w_{t+1}\leftarrow w_t-\eta_t g_t^w$ \COMMENT{gradient \emph{descent}}
    \STATE $v_{t+1}\leftarrow v_t+\kappa_v\eta_t g_t^v$ \COMMENT{gradient \emph{ascent}}
    \ENDFOR
    \STATE {\bfseries Output:} $(w_T,v_T)$
  \end{algorithmic}
\end{algorithm}

With harmonic stepsizes, we can state the global convergence guarantee of Algorithm~\ref{alg:SGDA} in the parameter space:

\begin{lem}[Global convergence of minibatch SGDA in parameter space]
\label{lem:sgda_global_conv_theta}
Suppose Conditions~\ref{asm:smooth}--\ref{asm:unique} hold \del{on the domain of the inner-PL geometry lemmas}\edit{on the effective domain $\Omega$ of Condition~\ref{asm:grad} (Section~\ref{sec:stability})}, and
use the conditionally unbiased minibatch updates above. 
Choose $\eta_t=c_1/(c_2+t)$ with $cc_1>1$ and
$c_2\ge\max\{1,c_1/\bar\eta\}$\edit{, where $c$ and $\bar\eta$ depend only on $L,\rho,\mu_{\mathrm{PL}}$ and $\kappa_v$}\edit{\footnote{$L_\Phi=L(1+L/\rho)$, $H=L_\Phi+L/2$, $c=\min\{\mu_{\mathrm{PL}}/2,\rho/2\}$, $\bar\eta=\min\{\frac1{4H},\frac1{4H(L/\rho)^2+L\kappa_v^2},\frac1c\}$, $C_{\mathrm{var}}'=(H+L\kappa_v^2/2)\sigma_B^2$, and $d_1:=\max\{c_2\E\Psi_{1,D}(w_0,v_0),\,C_{\mathrm{var}}'c_1^2/(cc_1-1)\}$. For comparisons and population bounds, take $d_1$ uniformly over the datasets under consideration. }}. Then\edit{, with $\Phi_D(w):=\max_vF_D(w,v)$ and $\Phi_D^\star:=\min_w\Phi_D(w)$,} the iterates converge almost surely
to an empirical minimax pair $(x_D^\star,v_D^\star)$, meaning
$x_D^\star\in\arg\min_w\Phi_D(w)$ and
$v_D^\star\in\arg\max_vF_D(x_D^\star,v)$. Moreover, with
$\Psi_{1,D}=\Phi_D-\Phi_D^\star+\Phi_D-F_D$ and $d_2:=c_2$,
\[
 \E[\Phi_D(w_t)-\Phi_D^\star]
 \le\E\Psi_{1,D}(w_t,v_t)\le\frac{d_1}{c_2+t}.
\]

\end{lem}
{The proof is given after Lemma~\ref{lem:coupled_corrected}. Lemma~\ref{lem:sgda_global_conv_theta} settles optimization; the rest of the paper bounds how the finite-sample error propagates to the population BRM objective.}

\section{Stability and Generalization for Bellman Residual Minimization}\label{sec:stability}
We quantify generalization through \emph{algorithmic stability} for minimax learning.
Algorithmic stability formalizes how sensitive a learning algorithm is to small changes
in the training set: if replacing a single training example only slightly perturbs the
algorithm's output, then the algorithm is said to be stable. Under the corresponding data-sampling assumptions, stability
implies generalization: algorithms that are stable on neighboring datasets exhibit small
discrepancies between their empirical risk and population risk. As shown in
\citet{Wang2022MCSGM}, this principle carries over to minimax optimization. \edit{Here generalization refers to the gap between the empirical and population BRM objectives, as in statistical learning theory \citep{bousquet2002stability,hardt2016train}; policy-level guarantees follow under coverage assumptions (Appendix~\ref{appssec:limitation}).}

\paragraph{Notations.} For the stability analysis, it is convenient to momentarily step back from the specific Bellman--residual objective and view our problem as a generic empirical minimax problem with primal and dual PL conditions. In the BRM formulation, the primal variable \(w\) parametrizes the action--value function \(Q_w\), the dual variable \(v\) parametrizes the auxiliary function \(\zeta_v\), and the empirical objective in \eqref{eq:EmpRM} is an average over samples \(z=(s,a,s')\) drawn from the offline dataset. In this section we abstract this structure and write \(f(w,v;z)\) for the per--sample saddle loss and \(F_D(w,v)\) for its empirical average over a dataset \(D=\{z_i\}_{i=1}^n\). Write $F(w,v):=\E[f(w,v;Z)]$ for the population objective under the target marginal distribution. In Theorem~\ref{thm:empRM_gen} we then specialize the resulting stability and generalization bounds back to BRM. Let $\mathcal D=(z_1,\dots,z_n)$ be a dataset with $z_i\in\mathcal Z$.
For any $i\in[n]$ and any $\tilde z_i\in\mathcal Z$, define the
\emph{replace-one neighbor} of $\mathcal D$ by
$
\mathcal D^{(i)}:=(z_1,\dots,z_{i-1},\tilde z_i,z_{i+1},\dots,z_n).
$
We call $\mathcal D$ and $\mathcal D^{(i)}$ \emph{neighboring datasets}. When comparing SGDA runs on $\mathcal D$ and $\mathcal D^{(i)}$, we couple them using the same minibatch index sequence \del{$(i_t)_{t\ge 0}$}\edit{$(I_t)_{t\ge 0}$} (shared-index coupling). Expectations averaged over $i$ are taken with $i\sim\mathrm{Unif}([n])$.
The conditional stability bound does not require an i.i.d. assumption on the data; population generalization is treated separately below. Appendix~\ref{appssec:index-pair} explains how the coupling works.

Throughout, we present the coupled updates for the \emph{single-sample} variant of SGDA. The extension to a minibatch of size $B\ge1$ replaces each sample gradient by its average. The conditional variance bound satisfies $\sigma_B^2\le G^2/B$, and it is zero for exact full-batch updates. The constants below retain this dependence.

\subsection{Stability}
We consider an offline setting where the data may be dependently sampled (e.g., from a single trajectory in a Markov Decision Process), violating the standard i.i.d. assumption. We use \textit{on-average algorithmic stability} to compare outputs on neighboring datasets; dependence enters the population bound through the covariance term below.

\begin{defn}[On-average algorithmic stability]\label{def:stability}
Let $\mathcal{A}$ be a randomized learning algorithm that maps a dataset
$\mathcal{D}=(z_1,\dots,z_n)$ to a parameter output $\mathcal{A}(\mathcal{D})$.
For each $i\in[n]$, let $\mathcal{D}^{(i)}$ denote the replace-one neighbor of $\mathcal{D}$,
where $z_i$ is replaced by an independent copy $\tilde z_i$.
Then the \emph{on-average argument stability} of $\mathcal{A}$ after $T$ iterations is
$
\varepsilon_T :=
\frac{1}{n}\sum_{i=1}^n
\E\bigl[\,
\|\mathcal{A}(\mathcal{D}) - \mathcal{A}(\mathcal{D}^{(i)})\|
\,\bigr],
$
where the expectation includes the algorithm's internal randomness and the replacement; population statements also average over the dataset.
\end{defn}
{Intuitively, $\varepsilon_T$ measures how much the algorithm's output changes when a single training point is replaced on average. We analyze the on-average argument stability of SGDA (i.e., $\mathcal{A}=$SGDA) for smooth minimax problems satisfying the stated PL conditions, following the stability framework for minimax optimization.}

\edit{We run Algorithm~\ref{alg:SGDA} on $\mathcal D$ and on $\mathcal D^{(i)}$ (where $z'_j=z_j$ for $j\neq i$ and $z'_i=\tilde z_i$) from the same initialization and with the same fresh uniform minibatches, so each stochastic gradient is conditionally unbiased for its own empirical objective given $\mathcal F_t:=\sigma(\mathcal D,\mathcal D^{(i)},w_0,v_0,I_0,\ldots,I_{t-1})$. For $D\in\{\mathcal D,\mathcal D^{(i)}\}$ write $F_D(w,v):=\frac1n\sum_{j=1}^nf(w,v;z_j^D)$, $\Phi_D(w):=\max_vF_D(w,v)$ and $\Phi_D^\star:=\min_w\Phi_D(w)$.}

Throughout, we impose the following conditions \ref{asm:smooth}--\ref{asm:unique} (satisfied by the BRM objective~\citep{kang2025empiricalriskminimizationapproach} under Section \ref{sec:setup} when the neural network parametrization's width is enough\footnote{For details, See Appendix \ref{appssec:Neural}.}):

\begin{enumerate}[label=(C\arabic*),leftmargin=2.6em,itemsep=2pt,topsep=2pt]
    \item \textbf{Smoothness.} $F_D$ is $L$-smooth in the joint variable $(w,v)$.\label{asm:smooth}
    \item \textbf{PL for $\Phi_D$.} $\Phi_D$ satisfies the Polyak--{\L}ojasiewicz (PL) inequality with parameter $\mu_{\mathrm{PL}}>0$: $\tfrac12\|\nabla\Phi_D(w)\|^2\ge \mu_{\mathrm{PL}}(\Phi_D(w)-\Phi_D^\star)$. \label{asm:PL-Phi}
    \item \textbf{QG for $\Phi_D$.} $\Phi_D$ satisfies Quadratic Growth (QG) with parameter $\mu_{\mathrm{QG}}>0$: $\Phi_D(w)-\Phi_D^\star\ge \tfrac{\mu_{\mathrm{QG}}}{2}\|w-x_D^\star\|^2$. \label{asm:QG-Phi}
    \item \textbf{\edit{Inner PL:} PL in $v$.} $F_D$ satisfies $\tfrac12\|\nabla_vF_D(w,v)\|^2\ge\rho\bigl(\Phi_D(w)-F_D(w,v)\bigr)$ uniformly in $(w,v)$. \label{asm:SC}
    \item \textbf{Bounded gradients on the effective domain.} \label{asm:grad}
There exists a compact convex set $\Omega\subset\mathcal W\times\mathcal V$ containing the iterates almost surely, the fixed population comparison pair, and the empirical maximizers and comparison segments used in the proofs.  We define $G<\infty$ as the uniform gradient bound on this set:
\[
G := \sup_{\substack{(w,v) \in \Omega, z \in \mathcal{Z}}} \|\nabla_w f(w,v; z)\|\vee \|\nabla_v f(w,v; z)\|.
\]
    \item \textbf{Attainment.} The inner maxima and empirical primal minima are attained, and the population primal risk has an attained minimizer. We write $x_D^\star\in\arg\min_w\Phi_D(w)$ and $V_D(w):=\arg\max_vF_D(w,v)$. \label{asm:unique}
\end{enumerate}

\begin{rem}[Clarifications on the conditions] The parameters are unconstrained, so $\Omega$ in~\ref{asm:grad} is a region that contains the iterates rather than a constraint set; because PL and smoothness hold only locally, the auxiliary ascent paths and segments used in the proofs are also taken in $\Omega$. Appendix~\ref{sec:A5proof} shows that such a set $\Omega$ exists under two-sided PL.
\end{rem} 

In addition, we impose the following settings for the stability analysis. Setting~\ref{asm:coupling} concerns the algorithm's random indices, whereas the observations in the offline dataset may be dependent. Uniformity is discussed further in Appendix~\ref{appssec:uniformity}.

\begin{enumerate}[label=(S\arabic*),leftmargin=2.6em,itemsep=2pt,topsep=2pt]
    \item \textbf{Stepsizes.} $0<\eta_t\le\bar\eta$, with
$\bar\eta$ as in Lemma~\ref{lem:sgda_global_conv_theta}; the dual stepsize
is $\kappa_v\eta_t$. \label{asm:stepsize}
\item \textbf{Shared-index coupling.} The two runs use the same fresh
uniform minibatch at each iteration, independently of the past.
No independence of the observations within the offline dataset is needed
for the conditional optimization and stability bounds. \label{asm:coupling}

\item \textbf{Uniformity of constants across datasets.} The constants $L,\rho,\mu_{\mathrm{PL}},\mu_{\mathrm{QG}},G$ are the same for $D$ and $D^{(i)}$. Population bounds use these constants and a common initial-potential bound uniformly over the sampled datasets; the bounds remain controlled as $n$ varies. \label{asm:uniform}
\end{enumerate}

To establish stability in our PL setting, a key difficulty is that SGDA is not optimizing the value function \(\Phi_D(w) := \max_v F_D(w,v)\) directly, but rather the saddle objective \(F_D(w,v)\). The gradient used for the primal update,
\(\nabla_w F_D(w_t,v_t)\), coincides with \(\nabla \Phi_D(w_t)\) whenever the dual variable \(v_t\) is already in its inner maximizer set \(V_D(w_t)\); away from this set there may be a ``mismatch'' term \(\Delta_t := \nabla_w F_D(w_t,v_t) - \nabla \Phi_D(w_t)\). As a result, tracking the PL suboptimality \(\Phi_D(w_t)-\Phi_D^\star\) alone does not yield a contracting recursion under SGDA, while tracking only the dual gap \(\Phi_D(w_t)-F_D(w_t,v_t)\) ignores how far the primal iterate is from the BRM minimizer.

To address this, we define the \textit{Lyapunov potential }:
\[
\Psi_{\alpha,D}(w,v)
:=\underbrace{\Phi_D(w)-\Phi_D^\star}_{A(w)}+\alpha\cdot\underbrace{\bigl(\Phi_D(w)-F_D(w,v)\bigr)}_{\,\Gamma(w,v)\,}.
\] \edit{The choice $\alpha=1$ cancels the mixed term in the one-step expansion, and the dual stepsize ratio $\kappa_v=1+(L/\rho)^2$ of Algorithm~\ref{alg:SGDA} lets dual ascent absorb the remaining squared gradient mismatch. The two PL inequalities then give contraction up to the additive stochastic-gradient error.}

The next theorem gives an $O(1/n)$ bound on the on-average primal argument stability of SGDA under the harmonic stepsizes of Lemma~\ref{lem:sgda_global_conv_theta}; Proposition~\ref{thm:sgda_stability_general} in Appendix~\ref{appendix: proof of main result} gives the bound for general stepsizes.

\begin{editblock}
\begin{thm}[On-average argument stability of SGDA]\label{thm:sgda_stability_noA7}
Suppose that the loss function $f$ and datasets $\mathcal{D}, \mathcal{D}^{(i)}$ satisfy \ref{asm:smooth}--\ref{asm:unique} under the setting \ref{asm:stepsize}--\ref{asm:uniform}, with $\alpha=1$, the dual stepsize multiplier $\kappa_v$ above and the harmonic stepsizes of Lemma~\ref{lem:sgda_global_conv_theta}. Let $\varepsilon_T^w:=\frac1n\sum_{i=1}^n\E\|w_T(\mathcal D)-w_T(\mathcal D^{(i)})\|$. Then, for $d_1,d_2$ from Lemma~\ref{lem:sgda_global_conv_theta} taken uniformly over the compared datasets,
\[
\varepsilon_T^w
\le 2C_{\mathrm{dist}}\sqrt{\frac{d_1}{d_2+T}}
    +\frac{C_{\mathrm{hit}}'}n
=O\bigl((d_2+T)^{-1/2}\bigr)+O(1/n),
\]
where $C_{\mathrm{dist}}$ and $C_{\mathrm{hit}}'$ depend only on $L,\rho,\mu_{\mathrm{PL}},\mu_{\mathrm{QG}},G$.\footnote{$C_{\mathrm{dist}}=\sqrt{2(1+L/\rho)^2/\mu_{\mathrm{QG}}+2/(\alpha\rho)}$ and $C_{\mathrm{hit}}'=2G\bigl((1+L/\rho)^2/\sqrt{\mu_{\mathrm{PL}}\mu_{\mathrm{QG}}}+1/\rho\bigr)$.} If, in addition, $V_D(x_D^\star)=\{v_D^\star\}$ for every compared dataset, the same bound holds for
$\varepsilon_T^{\mathrm{joint}}:=n^{-1}\sum_i\E[\|w_T-w'_T\|+\|v_T-v'_T\|]$; this condition is not needed for the primal stability or risk guarantees.
\end{thm}

\textit{Proof sketch.} Insert the two empirical primal minimizers between the coupled outputs:
\[
\|w_T-w'_T\|\le\underbrace{\|w_T-x^\star_{\mathcal D}\|+\|w'_T-x^\star_{\mathcal D^{(i)}}\|}_{\text{optimization}}
+\underbrace{\|x^\star_{\mathcal D}-x^\star_{\mathcal D^{(i)}}\|}_{\text{sensitivity}}.
\]
Three steps control these terms; the corresponding lemmas and proofs are in Appendix~\ref{appendix: Supporting Lemmas}.

\emph{Step 1 (contraction of the potential).} Write $g_t:=\nabla\Phi_D(w_t)$, $p_t:=\nabla_wF_D(w_t,v_t)=g_t+\Delta_t$ and $q_t:=\nabla_vF_D(w_t,v_t)$. In the one-step expansion of $\Psi_{1,D}$, the first-order term is $-2\langle g_t,p_t\rangle+\|p_t\|^2-\kappa_v\|q_t\|^2$, and the choice $\alpha=1$ cancels the mixed term: $-2\langle g_t,p_t\rangle+\|p_t\|^2=-\|g_t\|^2+\|\Delta_t\|^2$. Inner PL bounds the mismatch by $\|\Delta_t\|\le(L/\rho)\|q_t\|$, so the dual ratio $\kappa_v=1+(L/\rho)^2$ reduces the first-order term to $-\|g_t\|^2-\|q_t\|^2$. The two PL inequalities then give (Lemma~\ref{lem:coupled_corrected})
\[
\E[\Psi_{1,D}(w_{t+1},v_{t+1})\mid\mathcal F_t]\le(1-c\eta_t)\Psi_{1,D}(w_t,v_t)+C'_{\mathrm{var}}\eta_t^2,\qquad c=\min\{\mu_{\mathrm{PL}}/2,\rho/2\},
\]
and the harmonic stepsizes turn this recursion into $\E\Psi_{1,D}(w_T,v_T)\le d_1/(d_2+T)$ (Lemma~\ref{lem:sgda_global_conv_theta}).

\emph{Step 2 (from potential to distance).} Since $\Psi_{1,D}\ge\Phi_D-\Phi_D^\star$, QG gives (Lemma~\ref{lem:potential-to-distance})
\[
\|w-x^\star_D\|\le\sqrt{2\Psi_{1,D}(w,v)/\mu_{\mathrm{QG}}},
\]
so by Jensen's inequality each optimization term is at most $\sqrt{2d_1/(\mu_{\mathrm{QG}}(d_2+T))}$ in expectation.

\emph{Step 3 (minimizer sensitivity).} Replacing one sample moves each empirical gradient average by at most $2G/n$; with the inner PL error bound this gives $\|\nabla\Phi_{\mathcal D}(w)-\nabla\Phi_{\mathcal D^{(i)}}(w)\|\le2G(1+L/\rho)/n$ (Lemma~\ref{lem:phi-grad-diff}). At $u=x^\star_{\mathcal D^{(i)}}$, where $\nabla\Phi_{\mathcal D^{(i)}}(u)=0$, PL and QG give (Lemma~\ref{lem:sensitivity})
\[
\frac{\mu_{\mathrm{QG}}}{2}\|u-x^\star_{\mathcal D}\|^2\le\Phi_{\mathcal D}(u)-\Phi^\star_{\mathcal D}\le\frac{\|\nabla\Phi_{\mathcal D}(u)\|^2}{2\mu_{\mathrm{PL}}},
\]
hence $\|x^\star_{\mathcal D}-x^\star_{\mathcal D^{(i)}}\|\le S_w/n$ with $S_w:=2G(1+L/\rho)/\sqrt{\mu_{\mathrm{PL}}\mu_{\mathrm{QG}}}$.
Averaging over $i$ gives the theorem, since $C_{\mathrm{dist}}\ge\sqrt{2/\mu_{\mathrm{QG}}}$ and $C'_{\mathrm{hit}}\ge S_w$. The full proof is in Appendix~\ref{appendix: proof of main result}.

\begin{rem}[Where the conditions enter]\label{rem:conditions}
Inner PL \ref{asm:SC} gives the error bound for the auxiliary problem and the mismatch bound of Step~1, which the dual ratio $\kappa_v$ absorbs. Primal PL \ref{asm:PL-Phi} and inner PL give the contraction in Step~1, QG \ref{asm:QG-Phi} gives Step~2, and primal PL and QG together give the minimizer sensitivity in Step~3. PL therefore plays two roles, linear drift in the optimization recursion and control of minimizer sensitivity, and smoothness alone provides neither: the smooth nonconvex SGD stability bound of \citet{hardt2016train} grows with the number of iterations.
\end{rem}
\end{editblock}

\subsection{Generalization}

{In this section, we quantify generalization using the optimization and stability bounds derived in the previous section.} \edit{We bound the population BRM risk of the SGDA output through the primal risk below, which is the Bellman residual under auxiliary-function realizability.}

\begin{defn}[Primal Risk]\label{def:primalRisks}
Define the \emph{value functions}
\[
R(w):=\max_{v\in\mathcal V}F(w,v),\qquad
R_n(w):=\max_{v\in\mathcal V}F_{\mathcal D}(w,v)=\Phi_{\mathcal D}(w).
\]
\end{defn}

\edit{\textbf{From stability to generalization.} The generalization bound reuses the mechanism of Theorem~\ref{thm:sgda_stability_noA7}, with the population objective in place of the neighboring dataset. In Theorem~\ref{thm:sgda_stability_noA7}, replacing one sample perturbs $\nabla\Phi$ by $O(1/n)$, and Step~3 of its proof turns this into an $O(1/n)$ move of the empirical minimizer. The same step at the population minimizer $w^\star\in\arg\min_wR(w)$ gives $\|x^\star_{\mathcal D}-w^\star\|^2\le\|\nabla\Phi_{\mathcal D}(w^\star)\|^2/(\mu_{\mathrm{PL}}\mu_{\mathrm{QG}})$, which is $O(1/n)$ in expectation because $\nabla\Phi_{\mathcal D}(w^\star)$ is controlled by averages of $n$ mean-zero gradients. The optimization part is also shared: Steps~1 and~2 bound $\|w_T-x^\star_{\mathcal D}\|^2$. Since $\nabla R(w^\star)=0$, the excess risk is quadratic in $\|w_T-w^\star\|$, which gives the $O(1/n+1/T)$ rate and allows multiple dual maximizers. The resulting generalization bound is as follows.}

\begin{lem}[Generalization under inner PL]\label{lem:gen}
Let Conditions \ref{asm:smooth}--\ref{asm:unique} hold uniformly for the empirical objectives, and let \edit{the population objective $F$ also satisfy \ref{asm:smooth}, \ref{asm:SC} and inner attainment, so that $R$ is smooth and $\nabla R(w^\star)=0$}. For the population guarantees, assume that the observations are independently drawn from the population distribution\edit{\footnote{The stability bound of Theorem~\ref{thm:sgda_stability_noA7} needs no independence: the minibatch indices are drawn uniformly by the algorithm, so no mixing time enters it, unlike the Markov-chain indices of \citet{Wang2022MCSGM}. For stationary observations whose marginal matches the target distribution and whose fixed-pair gradient covariances are summable, the primal-risk statements here and in Theorems~\ref{thm:empRM_gen}--\ref{thm:pop_excess_risk} hold with $G^2$ replaced by $G^2C_{\mathrm{cov}}$ (Appendix~\ref{app:stationary}); $C_{\mathrm{cov}}$ plays the role of a mixing time.}}. Choose a fixed interior population minimizer $w^\star$ of $R$ and an interior maximizer $v^\star\in\arg\max_vF(w^\star,v)$. The common comparison domain must contain this pair, the required empirical maximizers, and the segments used below. Put
\[
a_T:=\E[\Phi_{\mathcal D}(w_T)-\Phi_{\mathcal D}^\star],\qquad
r_n:=\E\|\nabla\Phi_{\mathcal D}(w^\star)\|^2,\qquad
L_\Phi:=L(1+L/\rho).
\]
Then $r_n\le(1+L/\rho)^2G^2/n$ and
\[
\E[R(w_T)-R_n(w_T)]
\le\frac{2L_\Phi}{\mu_{\mathrm{QG}}}a_T
 +\left(\frac{L_\Phi}{\mu_{\mathrm{PL}}\mu_{\mathrm{QG}}}
              +\frac1{2\mu_{\mathrm{PL}}}\right)r_n.
\]
\end{lem}

Combining Lemma~\ref{lem:sgda_global_conv_theta} and Lemma~\ref{lem:gen}, we arrive at Theorem~\ref{thm:empRM_gen}, the \textit{generalization guarantee of Bellman residuals}. In words, the learned $Q$ (i.e., the corresponding learned $w$) generalizes: its expected population risk is bounded by its expected empirical risk plus the optimization error and an $\mathcal O(1/n)$ statistical term.

\begin{thm}[Generalization for the empirical Bellman--residual objective]\label{thm:empRM_gen}
Let $(\widehat w^{(T)},\widehat v^{(T)})$ be the parameters returned by Algorithm~\ref{alg:SGDA} after $T$ iterations on~\eqref{eq:EmpRM}, with the harmonic stepsizes of Lemma~\ref{lem:sgda_global_conv_theta} and under the conditions of Lemma~\ref{lem:gen}. \edit{Let $\mathcal R(w):=\max_v\mathbb E_{(s,a,s')}f(w,v;s,a,s')$ and $\widehat{\mathcal R}_n(w):=\max_v\frac1n\sum_{z\in\mathcal D_{\pi_D,\nu_0}}f(w,v;z)$ with $n=N$.}
Then, for $d_1,d_2$ from Lemma~\ref{lem:sgda_global_conv_theta},
\edit{\[
\mathbb E[\mathcal R(\widehat w^{(T)})-\widehat{\mathcal R}_n(\widehat w^{(T)})]
\le\frac{C_{\mathrm{opt}}\,d_1}{d_2+T}+\frac{C_{\mathrm{gen}}}{n},
\]
where $C_{\mathrm{opt}}$ and $C_{\mathrm{gen}}$ depend only on $L,\rho,\mu_{\mathrm{PL}},\mu_{\mathrm{QG}},G$.\footnote{$C_{\mathrm{opt}}=2L_\Phi/\mu_{\mathrm{QG}}$ and $C_{\mathrm{gen}}=(1+L/\rho)^2G^2\bigl(L_\Phi/(\mu_{\mathrm{PL}}\mu_{\mathrm{QG}})+1/(2\mu_{\mathrm{PL}})\bigr)$.}}
\end{thm}

The generalization guarantee in Theorem~\ref{thm:empRM_gen} bounds the expected population risk by the expected empirical risk plus the displayed error. The ultimate measure of success is the population risk relative to the best possible model. This is quantified by the \textit{population excess risk}, $\mathcal R(\widehat w^{(T)})-\mathcal R(w^\star)$.

\begin{thm}[Population excess risk]\label{thm:pop_excess_risk}
Under the conditions of Theorem~\ref{thm:empRM_gen}, let $w^\star\in\arg\min_w\mathcal R(w)$. Then
\edit{\[
\mathbb E[\mathcal R(\widehat w^{(T)})-\mathcal R(w^\star)]
\le\frac{C_{\mathrm{opt}}\,d_1}{d_2+T}+\frac{C_{\mathrm{stat}}}{n},
\qquad C_{\mathrm{stat}}:=\frac{L_\Phi(1+L/\rho)^2G^2}{\mu_{\mathrm{PL}}\mu_{\mathrm{QG}}},
\]
where the first term is the optimization error and the second the statistical error.}
\end{thm}

\begin{editblock}
\begin{rem}[Weak primal--dual gap]\label{rem:weakPD}
The primal-risk bounds above borrow the structure of the proof of Theorem~\ref{thm:sgda_stability_noA7}. The weak primal--dual gap of \citet{Wang2022MCSGM}, $\Delta^{\mathrm{PD}}(w_T,v_T):=\max_{v'}\E F(w_T,v')-\min_{w'}\E F(w',v_T)$, compares the output with fixed comparators on each side, and it follows directly from stability: for independent observations, the exchange argument of \citet{bousquet2002stability} gives $|\Delta^{\mathrm{PD}}-\Delta^{\mathrm{PD}}_n|\le G\varepsilon_T^{\mathrm{joint}}$, where $\Delta^{\mathrm{PD}}_n$ is its empirical counterpart, and the joint statement of Theorem~\ref{thm:sgda_stability_noA7} makes this $O((d_2+T)^{-1/2})+O(1/n)$. See Appendix~\ref{app:weakPD}.
\end{rem}
\end{editblock}

\section[Numerical Illustration]{\edit{Numerical Illustration}}\label{sec:exp}

\begin{figure}[t]
\centering
\includegraphics[width=\textwidth]{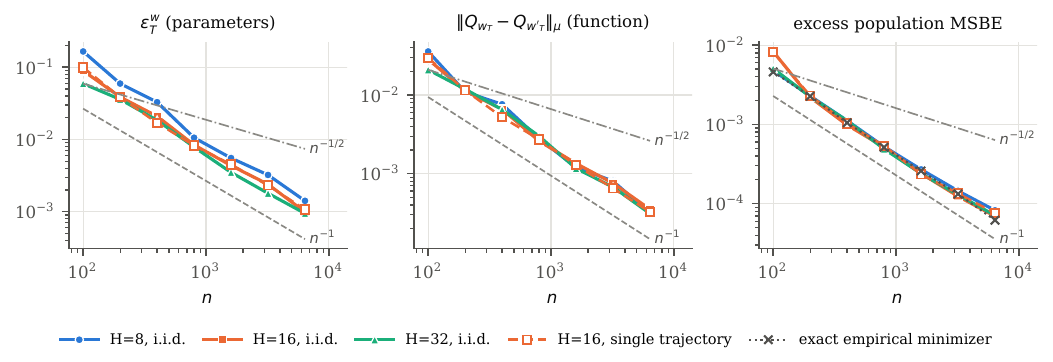}
\caption{\edit{Primal argument stability, function-space stability and excess population MSBE of the SGDA output against $n$ (log--log), averaged over 50 repetitions. Solid lines: i.i.d.\ data; dashed: a single trajectory; dotted: exact empirical minimizer. Grey guides show slopes $-1$ and $-1/2$.}}
\label{fig:exp}
\end{figure}

\begin{editblock}
We check the $n^{-1}$ rates of Theorems~\ref{thm:sgda_stability_noA7} and~\ref{thm:pop_excess_risk} on a small soft MDP ($|\mathcal S|=5$, $|\mathcal A|=3$, $\beta=0.9$) in which all population quantities are exact. Two-layer softplus networks of width $H\in\{8,16,32\}$ are trained by Algorithm~\ref{alg:SGDA} on $n\in\{100,\dots,6400\}$ transitions, i.i.d.\ or from a single trajectory, and each run on $\mathcal D$ is coupled with a run on a neighbor $\mathcal D^{(i)}$ through the same initialization and minibatch indices (details in Appendix~\ref{app:experiments}). Code is available at \url{https://github.com/jajajang/soft-brm-stability}.

Both stability and excess MSBE decay as $n^{-1}$ (Figure~\ref{fig:exp}): over $n\ge200$ the fitted slopes are $-1.08$ to $-1.05$ for $\varepsilon^w_T$ and $-1.01$ to $-0.97$ for the excess MSBE, and every $95\%$ bootstrap interval excludes $-1/2$. A single dependent trajectory gives the same slopes, as the stability bound needs no independence. The SGDA output is within $10\%$ of the exact empirical minimizer for $200\le n\le3200$, and the width changes the constant\footnote{Appendix~\ref{app:experiments} also compares the constants in these theorems with the observed values.} but not the rate. In short, both the stability of SGDA and the excess MSBE of its output follow the $n^{-1}$ rates predicted by Theorems~\ref{thm:sgda_stability_noA7} and~\ref{thm:pop_excess_risk}.
\end{editblock}

\section[Conclusion and Limitations]{Conclusion and Limitation\edit{s}}
\edit{We studied the statistical behavior of Bellman Residual Minimization (BRM) in offline RL/IRL through the lens of stability. Using the primal and dual PL geometry of the bi-conjugate formulation, a single Lyapunov potential and empirical-minimizer sensitivity give an $O(1/n)$ on-average primal argument-stability term for SGDA (Theorem~\ref{thm:sgda_stability_noA7}), $O(1/n)$ generalization for BRM (Theorem~\ref{thm:empRM_gen}), and a population excess-risk bound that separates optimization and statistical error (Theorem~\ref{thm:pop_excess_risk}). The analysis needs no variance reduction, and stationary dependence enters only through a covariance factor.}

\edit{\textbf{Implications for offline RL.} Theorem~\ref{thm:pop_excess_risk} answers the question in the Introduction: an excess MSBE of $\epsilon$ requires $n=O(C_{\mathrm{stat}}/\epsilon)$ transitions and a number of SGDA iterations of order $n$. Beyond that point, longer optimization on the same dataset cannot improve the guarantee, which is then limited by the statistical term.}

 \paragraph{Limitations and future work.} This work focuses on finite-sample stability and generalization guarantees for smooth, entropy-regularized BRM, and therefore does not address genuinely non-smooth BRM variants or the conversion from Bellman-residual guarantees to policy-performance guarantees without additional coverage assumptions. We discuss these topics further in Appendix~\ref{appssec:limitation}, and view them as natural next steps toward a more complete statistical theory of BRM-based offline reinforcement learning.
In addition, our theory guarantees the PL conditions only for relatively wide neural networks (Appendix~\ref{appssec:Neural}), and the geometry of soft BRM at practical network sizes requires further study.

\begin{ack}
K.~Jang was supported by the Institute of Information \& Communications Technology Planning \& Evaluation (IITP) grant funded by the Korea government (MSIT) [RS-2021-II211341, Artificial Intelligence Graduate School Program (Chung-Ang University)] and by the National Research Foundation of Korea(NRF) grant funded by the Korea government (MSIT). The authors declare no competing interests.
\end{ack}

\bibliographystyle{plainnat}
\bibliography{bib}

@article{liu2022loss,
  title={Loss landscapes and optimization in over-parameterized non-linear systems and neural networks},
  author={Liu, Chaoyue and Zhu, Libin and Belkin, Mikhail},
  journal={Applied and Computational Harmonic Analysis},
  volume={59},
  pages={85--116},
  year={2022},
  publisher={Elsevier}
}

@article{bousquet2002stability,
  title={Stability and generalization},
  author={Bousquet, Olivier and Elisseeff, Andr{\'e}},
  journal={Journal of machine learning research},
  volume={2},
  number={Mar},
  pages={499--526},
  year={2002}
}

@article{feldman2018generalization,
  title={Generalization bounds for uniformly stable algorithms},
  author={Feldman, Vitaly and Vondrak, Jan},
  journal={Advances in Neural Information Processing Systems},
  volume={31},
  year={2018}
}

@inproceedings{charles2018stability,
  title={Stability and generalization of learning algorithms that converge to global optima},
  author={Charles, Zachary and Papailiopoulos, Dimitris},
  booktitle={International conference on machine learning},
  pages={745--754},
  year={2018},
  organization={PMLR}
}

@inproceedings{Lei2021SGDAStability,
  title     = {Stability and Generalization of Stochastic Gradient Methods for Minimax Problems},
  author    = {Lei, Yunwen and Yang, Zhenhuan and Yang, Tianbao and Ying, Yiming},
  booktitle = {Proceedings of the 38th International Conference on Machine Learning},
  series    = {Proceedings of Machine Learning Research},
  volume    = {139},
  pages     = {6175--6186},
  year      = {2021},
  editor    = {Meil{\u{a}}, Marina and Zhang, Tong},
  publisher = {PMLR},
  url       = {https://proceedings.mlr.press/v139/lei21b.html}
}

@inproceedings{Wang2022MCSGM,
  title     = {Stability and Generalization for Markov Chain Stochastic Gradient Methods},
  author    = {Wang, Puyu and Lei, Yunwen and Ying, Yiming and Zhou, Ding-Xuan},
  booktitle = {Advances in Neural Information Processing Systems},
  volume    = {35},
  year      = {2022},
  url       = {https://proceedings.neurips.cc/paper_files/paper/2022/hash/f61538f83b0f19f9306d9d801c15f41c-Abstract-Conference.html}
}

@article{Karimi2016PL,
  title   = {Linear Convergence of Gradient and Proximal-Gradient Methods under the Polyak-{\L}ojasiewicz Condition},
  author  = {Karimi, Hamed and Nutini, Julie and Schmidt, Mark},
  journal = {arXiv preprint arXiv:1608.04636},
  year    = {2016},
  url     = {https://arxiv.org/abs/1608.04636}
}

@misc{GarrigosGower2023Handbook,
  title        = {Handbook of Convergence Theorems for (Stochastic) Gradient Methods},
  author       = {Garrigos, Guillaume and Gower, Robert M.},
  year         = {2023},
  howpublished = {arXiv preprint arXiv:2301.11235},
  url          = {https://arxiv.org/abs/2301.11235}
}

@misc{kang2025empiricalriskminimizationapproach,
      title={An Empirical Risk Minimization Approach for Offline Inverse RL and Dynamic Discrete Choice Model}, 
      author={Enoch H. Kang and Hema Yoganarasimhan and Lalit Jain},
      year={2025},
      eprint={2502.14131},
      archivePrefix={arXiv},
      primaryClass={cs.LG},
      url={https://arxiv.org/abs/2502.14131}, 
}

@inproceedings{dai2018sbeed,
  title={SBEED: Convergent reinforcement learning with nonlinear function approximation},
  author={Dai, Bo and Shaw, Albert and Li, Lihong and Xiao, Lin and He, Niao and Liu, Zhen and Chen, Jianshu and Song, Le},
  booktitle={International conference on machine learning},
  pages={1125--1134},
  year={2018},
  organization={PMLR}
}

@article{patterson2022generalized,
  title={A generalized projected bellman error for off-policy value estimation in reinforcement learning},
  author={Patterson, Andrew and White, Adam and White, Martha},
  journal={Journal of Machine Learning Research},
  volume={23},
  number={145},
  pages={1--61},
  year={2022}
}

@misc{yang2020globalconvergencevariancereducedoptimization,
      title={Global Convergence and Variance-Reduced Optimization for a Class of Nonconvex-Nonconcave Minimax Problems}, 
      author={Junchi Yang and Negar Kiyavash and Niao He},
      year={2020},
      eprint={2002.09621},
      archivePrefix={arXiv},
      primaryClass={math.OC},
      url={https://arxiv.org/abs/2002.09621}, 
}

@article{antos2008learning,
  title={Learning near-optimal policies with Bellman-residual minimization based fitted policy iteration and a single sample path},
  author={Antos, Andr{\'a}s and Szepesv{\'a}ri, Csaba and Munos, R{\'e}mi},
  journal={Machine Learning},
  volume={71},
  number={1},
  pages={89--129},
  year={2008},
  publisher={Springer}
}

@article{ernst2005tree,
  title={Tree-based batch mode reinforcement learning},
  author={Ernst, Damien and Geurts, Pierre and Wehenkel, Louis},
  journal={Journal of Machine Learning Research},
  volume={6},
  year={2005},
  publisher={Microtome Publishing, Brookline, United States-Massachusetts}
}

@article{rust1994structural,
  title={Structural estimation of Markov decision processes},
  author={Rust, John},
  journal={Handbook of econometrics},
  volume={4},
  pages={3081--3143},
  year={1994},
  publisher={Elsevier}
}

@article{jiang2024offline,
  title={Offline reinforcement learning in large state spaces: Algorithms and guarantees},
  author={Jiang, Nan and Xie, Tengyang},
  journal={Statistical Science},
  year={2025},
  note={arXiv:2510.04088}
}

@article{van2018deep,
  title={Deep reinforcement learning and the deadly triad},
  author={Van Hasselt, Hado and Doron, Yotam and Strub, Florian and Hessel, Matteo and Sonnerat, Nicolas and Modayil, Joseph},
  journal={arXiv preprint arXiv:1812.02648},
  year={2018}
}

@article{tsitsiklis1996feature,
  title={Feature-based methods for large scale dynamic programming},
  author={Tsitsiklis, John N and Van Roy, Benjamin},
  journal={Machine Learning},
  volume={22},
  number={1},
  pages={59--94},
  year={1996},
  publisher={Springer}
}

@inproceedings{hardt2016train,
  title={Train faster, generalize better: Stability of stochastic gradient descent},
  author={Hardt, Moritz and Recht, Ben and Singer, Yoram},
  booktitle={International conference on machine learning},
  pages={1225--1234},
  year={2016},
  organization={PMLR}
}

@inproceedings{baird1995residual,
  title={Residual Algorithms: Reinforcement Learning with Function Approximation},
  author={Baird, Leemon},
  booktitle={Machine Learning Proceedings 1995},
  pages={30--37},
  year={1995},
  publisher={Elsevier}
}

@inproceedings{liu2018breaking,
  title={Breaking the Curse of Horizon: Infinite-Horizon Off-Policy Estimation},
  author={Liu, Qiang and Li, Lihong and Tang, Ziyang and Zhou, Dengyong},
  booktitle={Advances in Neural Information Processing Systems},
  year={2018}
}

@inproceedings{kuzborskij2018data,
  title={Data-Dependent Stability of Stochastic Gradient Descent},
  author={Kuzborskij, Ilja and Lampert, Christoph H.},
  booktitle={Proceedings of the 35th International Conference on Machine Learning},
  pages={2815--2824},
  year={2018},
  organization={PMLR}
}

@inproceedings{bousquet2020sharper,
  title={Sharper Bounds for Uniformly Stable Algorithms},
  author={Bousquet, Olivier and Klochkov, Yegor and Zhivotovskiy, Nikita},
  booktitle={Proceedings of Thirty Third Conference on Learning Theory},
  pages={610--626},
  year={2020},
  organization={PMLR}
}

@inproceedings{farnia2021train,
  title={Train Simultaneously, Generalize Better: Stability of Gradient-Based Minimax Learners},
  author={Farnia, Farzan and Ozdaglar, Asuman},
  booktitle={Proceedings of the 38th International Conference on Machine Learning},
  pages={3174--3185},
  year={2021},
  organization={PMLR}
}

@article{polyak1963gradient,
  title={Gradient Methods for Minimizing Functionals},
  author={Polyak, Boris T.},
  journal={USSR Computational Mathematics and Mathematical Physics},
  volume={3},
  number={4},
  pages={864--878},
  year={1963},
  publisher={Elsevier}
}

@inproceedings{jacot2018neural,
  title={Neural Tangent Kernel: Convergence and Generalization in Neural Networks},
  author={Jacot, Arthur and Gabriel, Franck and Hongler, Cl{\'e}ment},
  booktitle={Advances in Neural Information Processing Systems},
  volume={31},
  year={2018}
}

@inproceedings{du2018gradient,
  title={Gradient Descent Provably Optimizes Over-Parameterized Neural Networks},
  author={Du, Simon S. and Zhai, Xiyu and Poczos, Barnabas and Singh, Aarti},
  booktitle={International Conference on Learning Representations},
  year={2019}
}

@inproceedings{allen2019convergence,
  title={A Convergence Theory for Deep Learning via Over-Parameterization},
  author={Allen-Zhu, Zeyuan and Li, Yuanzhi and Song, Zhao},
  booktitle={Proceedings of the 36th International Conference on Machine Learning},
  pages={242--252},
  year={2019},
  organization={PMLR}
}

@inproceedings{arora2019fine,
  title={Fine-Grained Analysis of Optimization and Generalization for Overparameterized Two-Layer Neural Networks},
  author={Arora, Sanjeev and Du, Simon S. and Hu, Wei and Li, Zhiyuan and Wang, Ruosong},
  booktitle={Proceedings of the 36th International Conference on Machine Learning},
  pages={322--332},
  year={2019},
  organization={PMLR}
}

@inproceedings{lee2019wide,
  title={Wide Neural Networks of Any Depth Evolve as Linear Models under Gradient Descent},
  author={Lee, Jaehoon and Xiao, Lechao and Schoenholz, Samuel S. and Bahri, Yasaman and Novak, Roman and Sohl-Dickstein, Jascha and Pennington, Jeffrey},
  booktitle={Advances in Neural Information Processing Systems},
  volume={32},
  year={2019}
}

@article{soltanolkotabi2018theoretical,
  title={Theoretical Insights into the Optimization Landscape of Over-Parameterized Shallow Neural Networks},
  author={Soltanolkotabi, Mahdi and Javanmard, Adel and Lee, Jason D.},
  journal={IEEE Transactions on Information Theory},
  volume={65},
  number={2},
  pages={742--769},
  year={2019},
  publisher={IEEE}
}

@article{lei2020fine,
  title={Fine-Grained Analysis of Stability and Generalization for Stochastic Gradient Descent},
  author={Lei, Yunwen and Ying, Yiming},
  journal={Proceedings of the 37th International Conference on Machine Learning},
  pages={5809--5819},
  year={2020},
  organization={PMLR}
}

@book{sutton2018reinforcement,
  title={Reinforcement Learning: An Introduction},
  author={Sutton, Richard S. and Barto, Andrew G.},
  edition={2},
  year={2018},
  publisher={MIT Press}
}

@inproceedings{nachum2019dualdice,
  title={DualDICE: Behavior-Agnostic Estimation of Discounted Stationary Distribution Corrections},
  author={Nachum, Ofir and Chow, Yinlam and Dai, Bo and Li, Lihong},
  booktitle={Advances in Neural Information Processing Systems},
  volume={32},
  year={2019}
}

@inproceedings{zhang2020gendice,
  title={GenDICE: Generalized Offline Estimation of Stationary Values},
  author={Zhang, Ruiyi and Dai, Bo and Li, Lihong and Schuurmans, Dale},
  booktitle={International Conference on Learning Representations},
  year={2020}
}

@inproceedings{zhang2024generalization,
  title={Generalization Bounds of Nonconvex-(Strongly)-Concave Stochastic Minimax Optimization},
  author={Zhang, Siqi and Hu, Yifan and Zhang, Liang and He, Niao},
  booktitle={Proceedings of The 27th International Conference on Artificial Intelligence and Statistics},
  pages={694--702},
  year={2024},
  volume={238},
  series={Proceedings of Machine Learning Research},
  publisher={PMLR}
}

@inproceedings{zhu2023stability,
  title={Stability and Generalization of the Decentralized Stochastic Gradient Descent Ascent Algorithm},
  author={Zhu, Miaoxi and Shen, Li and Du, Bo and Tao, Dacheng},
  booktitle={Advances in Neural Information Processing Systems},
  volume={36},
  year={2023}
}

@inproceedings{chen2023overparameterized,
  title={Over-Parameterized Model Optimization with Polyak-{\L}ojasiewicz Condition},
  author={Chen, Yixuan and Shi, Yubin and Dong, Mingzhi and Yang, Xiaochen and Li, Dongsheng and Wang, Yujiang and Dick, Robert P. and Lv, Qin and Zhao, Yingying and Yang, Fan and Gu, Ning and Shang, Li},
  booktitle={International Conference on Learning Representations},
  year={2023}
}

@article{xu2025local,
  title={A Local Polyak-{\L}ojasiewicz and Descent Lemma of Gradient Descent for Overparametrized Linear Models},
  author={Xu, Ziqing and Min, Hancheng and Tarmoun, Salma and Mallada, Enrique and Vidal, Ren{\'e}},
  journal={Transactions on Machine Learning Research},
  year={2025}
}

@inproceedings{maillard2010finite,
  title     = {Finite-Sample Analysis of {B}ellman Residual Minimization},
  author    = {Maillard, Odalric-Ambrym and Munos, R{\'e}mi and Lazaric, Alessandro and Ghavamzadeh, Mohammad},
  booktitle = {Proceedings of the 2nd Asian Conference on Machine Learning},
  series    = {Proceedings of Machine Learning Research},
  volume    = {13},
  pages     = {299--314},
  year      = {2010}
}

@inproceedings{chen2019information,
  title     = {Information-Theoretic Considerations in Batch Reinforcement Learning},
  author    = {Chen, Jinglin and Jiang, Nan},
  booktitle = {Proceedings of the 36th International Conference on Machine Learning},
  series    = {Proceedings of Machine Learning Research},
  volume    = {97},
  pages     = {1042--1051},
  year      = {2019}
}

@inproceedings{duan2021risk,
  title     = {Risk Bounds and {R}ademacher Complexity in Batch Reinforcement Learning},
  author    = {Duan, Yaqi and Jin, Chi and Li, Zhiyuan},
  booktitle = {Proceedings of the 38th International Conference on Machine Learning},
  series    = {Proceedings of Machine Learning Research},
  volume    = {139},
  pages     = {2892--2902},
  year      = {2021}
}

@article{hu2025fast,
  title   = {Fast Rates for the Regret of Offline Reinforcement Learning},
  author  = {Hu, Yichun and Kallus, Nathan and Uehara, Masatoshi},
  journal = {Mathematics of Operations Research},
  volume  = {50},
  number  = {1},
  pages   = {633--655},
  year    = {2025},
  doi     = {10.1287/moor.2021.0167}
}

@article{zhao2009reinforcement,
  title={Reinforcement learning design for cancer clinical trials},
  author={Zhao, Yufan and Kosorok, Michael R. and Zeng, Donglin},
  journal={Statistics in Medicine},
  volume={28},
  number={26},
  pages={3294--3315},
  year={2009}
}

@article{afsar2022reinforcement,
  title={Reinforcement learning based recommender systems: A survey},
  author={Afsar, M. Mehdi and Crump, Trafford and Far, Behrouz},
  journal={ACM Computing Surveys},
  volume={55},
  number={7},
  pages={1--38},
  year={2022}
}

@inproceedings{bassen2020reinforcement,
  title={Reinforcement learning for the adaptive scheduling of educational activities},
  author={Bassen, Jonathan and Balaji, Bharathan and Schaarschmidt, Michael and Thille, Candace and Painter, Jay and Zimmaro, Dawn and Games, Alex and Fast, Ethan and Mitchell, John C.},
  booktitle={Proceedings of the 2020 CHI Conference on Human Factors in Computing Systems},
  pages={1--12},
  year={2020}
}

\newpage
\appendix
\onecolumn
\section*{Appendix}
\section{Notations}
\begin{table}[h]
    \centering
    \begin{tabular}{c|c}
        \toprule
    Symbol     & Meaning \\
    \midrule
    $\mathcal{S}$     & State space\\
    $\mathcal{A}$   & Action space\\
    $P(\cdot|s,a)$ & Transition kernel over the next state from the state-action pair $(s,a)$\\
    $r: \mathcal{S} \times \mathcal{A} \to \mathbb{R}$ & Deterministic reward function\\
    $\Delta_{\mathcal{A}}^{\mathcal{S}}$ & Set of stationary Markov policies\\
    $\pi \in \Delta_{\mathcal{A}}^{\mathcal{S}}$ & Policy\\
    $\mathcal{H}(q)$ & Shannon entropy of an action distribution $q$\\
    $\beta$ & Discount factor\\
    $\eta$ & Entropy regularization coefficient\\
    $\mathcal{T}$ & Bellman optimality operator\\
    $\hat{\mathcal{T}}$ & Sampled Bellman operator\\
    $\delta_Q(s,a,s')$ & Temporal-Difference error\\
    $Q^*$ & (The unique) solution to Bellman optimality equation\\
    $z_i =(s_i , a_i , s_i')$ & $i$-th sample\\
    $\mathcal{D}=\{z_i \}_{i=1}^n$ & Dataset\\
    $f(w,v;z_i)$ & Per-sample objective\\
    $\mathcal{D}^{(i)}$ & neighboring dataset\\
    $g_t^w, g_t^v$ & Minibatch gradients\\
    $\eta_t$ & SGDA learning rate at time $t$\\
    $F_D(w,v)$ & $\tfrac1n\sum_{j=1}^n f(w,v;z_j^D)$\\
    $\Phi_D(w)$ & $\max_v F_D(w,v)$\\
    $\Phi_D^\star$ & $\min_w \Phi_D(w)$\\
    $L$ & Smoothness (Condition \ref{asm:smooth})\\
    $\mu_{\mathrm{PL}}$ & Polyak--{\L}ojasiewicz condition constant (Condition \ref{asm:PL-Phi})\\
    $\mu_{\mathrm{QG}}$ & Quadratic Growth constant (Condition \ref{asm:QG-Phi})\\
    $\rho$ & Inner PL constant (Condition \ref{asm:SC})\\
    $G$ & Per-sample gradient bound (Condition \ref{asm:grad})\\
$\Psi_{\alpha,D}(w,v)$ & Lyapunov potential, 
$\Phi_D(w)-\Phi_D^\star+\alpha\cdot\bigl(\Phi_D(w)-F_D(w,v)\bigr)$\\
$A(w)$ & $\Phi_D(w)-\Phi_D^\star$\\
$\Gamma(w,v)$ & $\Phi_D(w)-F_D(w,v)$\\
$R(w), R_n(w)$ & Primal Risk (Definition \ref{def:primalRisks})\\
$\Delta^{\mathrm{PD}}$, $\Delta_n^{\mathrm{PD}}$ & Weak primal--dual risk (Definition \ref{def:weakPD})\\
        \bottomrule
    \end{tabular}
    \caption{Notations}
    \label{tab:notations}
\end{table}

\section{Technical Proofs}

\subsection{Supporting Lemmas for proving Theorem \ref{thm:sgda_stability_noA7}} \label{appendix: Supporting Lemmas}

\begin{lem}[PL error bound]\label{lem:inner-error-bound}
Under Conditions~\ref{asm:smooth} and~\ref{asm:SC}, let
\[
V_D(w):=\arg\max_v F_D(w,v),\qquad
\Gamma_D(w,v):=\Phi_D(w)-F_D(w,v).
\]
Then
\[
\operatorname{dist}(v,V_D(w))
\leq\sqrt{\frac{2\Gamma_D(w,v)}{\rho}}
\leq\frac{\|\nabla_vF_D(w,v)\|}{\rho}.
\]
\end{lem}
\begin{proof}
Fix $w$ and follow $\dot v=\nabla_vF_D(w,v)$ from $v(0)=v$.
Writing $q(t)=\Gamma_D(w,v(t))$, inner PL gives
$\dot q=-\|\dot v\|^2\leq-2\rho q$ and
$-\frac{d}{dt}\sqrt q\geq\sqrt{\rho/2}\,\|\dot v\|$.
Integration shows that the trajectory has length at most
$\sqrt{2q(0)/\rho}$ and converges to a point of $V_D(w)$.
This proves the first inequality; inner PL proves the second.
The domain conditions ensure that this auxiliary ascent trajectory
remains in the region on which the inequalities hold.
\end{proof}

\begin{lem}[Smoothness of the value function $\Phi$]\label{lem:value-smooth}
Assume Conditions~\ref{asm:smooth} and~\ref{asm:SC}.  Then:
\begin{enumerate}[label=(\roman*),itemsep=2pt,topsep=2pt]
\item For every $p\in V_D(w)$ there is $q\in V_D(u)$ such that
$\|p-q\|\leq(L/\rho)\|w-u\|$.
\item $\Phi_D$ is differentiable and
$\nabla\Phi_D(w)=\nabla_wF_D(w,p)$ for every $p\in V_D(w)$.
\item $\Phi_D$ is $L_\Phi$-smooth, where $L_\Phi:=L(1+L/\rho)$.
\end{enumerate}
\end{lem}
\begin{proof}
Since $\nabla_vF_D(w,p)=0$, Lemma~\ref{lem:inner-error-bound}
gives
\[
\operatorname{dist}(p,V_D(u))
\leq\rho^{-1}\|\nabla_vF_D(u,p)\|
\leq(L/\rho)\|u-w\|,
\]
proving (i).  Choose such a $q$, and put $h=u-w$.
Using $F_D(w,p)=\Phi_D(w)$ and $F_D(w,q)\leq\Phi_D(w)$,
smoothness yields
\begin{align*}
\langle\nabla_wF_D(w,p),h\rangle-\tfrac L2\|h\|^2
&\leq\Phi_D(u)-\Phi_D(w)\\
&\leq\langle\nabla_wF_D(w,p),h\rangle
 +(\tfrac L2+L^2/\rho)\|h\|^2.
\end{align*}
Thus (ii) holds for every $p$.  Finally,
\[
\|\nabla\Phi_D(w)-\nabla\Phi_D(u)\|
\leq L(\|w-u\|+\|p-q\|)
\leq L_\Phi\|w-u\|.
\]
\end{proof}

\begin{lem}[Mismatch control]\label{lem:mismatch}
Under Conditions~\ref{asm:smooth} and~\ref{asm:SC},
\[
\|\Delta(w,v)\|
:=\|\nabla_wF_D(w,v)-\nabla\Phi_D(w)\|
\leq L\operatorname{dist}(v,V_D(w))
\leq L\sqrt{\frac{2\Gamma_D(w,v)}{\rho}}.
\]
\end{lem}
\begin{proof}
Choose a closest $p\in V_D(w)$, apply
Lemma~\ref{lem:value-smooth}(ii) and joint smoothness, and then
apply Lemma~\ref{lem:inner-error-bound}.
\end{proof}

\begin{lem}[Cross-dataset gradient sensitivity for $\Phi$]\label{lem:phi-grad-diff}
Under Conditions~\ref{asm:smooth}, \ref{asm:SC}, \ref{asm:grad},
and~\ref{asm:uniform}, a one-sample replacement $D\to D^{(i)}$ gives
\[
\|\nabla\Phi_D(w)-\nabla\Phi_{D^{(i)}}(u)\|
\leq L_\Phi\|w-u\|+\frac{2G}{n}(1+L/\rho).
\]
\end{lem}
\begin{proof}
Fix $p'\in V_{D^{(i)}}(u)$.  The replacement identity and
Condition~\ref{asm:grad} give
$\|\nabla_vF_D(u,p')\|\leq2G/n$.  By
Lemma~\ref{lem:inner-error-bound}, there is $p\in V_D(u)$ with
$\|p-p'\|\leq2G/(\rho n)$.  Therefore
\[
\|\nabla\Phi_D(u)-\nabla\Phi_{D^{(i)}}(u)\|
\leq L\|p-p'\|+2G/n
\leq\frac{2G}{n}(1+L/\rho).
\]
Add $\|\nabla\Phi_D(w)-\nabla\Phi_D(u)\|\leq L_\Phi\|w-u\|$.
\end{proof}

\begin{lem}[Deterministic potential $\Rightarrow$ distance]\label{lem:potential-to-distance}
Under Condition~\ref{asm:QG-Phi}, for every $\alpha>0$,
\[
\|w-x_D^\star\|\leq
\sqrt{\frac{2\Psi_{\alpha,D}(w,v)}{\mu_{\mathrm{QG}}}}.
\]
If, additionally, $V_D(x_D^\star)=\{v_D^\star\}$ and
Conditions~\ref{asm:smooth} and~\ref{asm:SC} hold, then
\[
\|w-x_D^\star\|+\|v-v_D^\star\|
\leq C_{\mathrm{dist}}\sqrt{\Psi_{\alpha,D}(w,v)},\qquad
C_{\mathrm{dist}}:=
\sqrt{\frac{2(1+L/\rho)^2}{\mu_{\mathrm{QG}}}
                  +\frac{2}{\alpha\rho}}.
\]
\end{lem}
\begin{proof}
The primal claim follows from QG and
$\Psi_{\alpha,D}\geq A_D:=\Phi_D-\Phi_D^\star$.
For the joint claim choose a closest $p\in V_D(w)$ to $v$.
Lemma~\ref{lem:value-smooth}(i) and the singleton condition give
$\|p-v_D^\star\|\leq(L/\rho)\|w-x_D^\star\|$.  Hence
\[
\|w-x_D^\star\|+\|v-v_D^\star\|
\leq(1+L/\rho)\sqrt{2A_D(w)/\mu_{\mathrm{QG}}}
          +\sqrt{2\Gamma_D(w,v)/\rho}.
\]
Weighted Cauchy--Schwarz gives the displayed constant.
\end{proof}

\begin{lem}[Primal and saddle-point sensitivity]\label{lem:sensitivity}
Under Conditions~\ref{asm:smooth}--\ref{asm:grad} and
Setting~\ref{asm:uniform},
\[
\|x_D^\star-x_{D^{(i)}}^\star\|\leq\frac{S_w}{n},\qquad
S_w:=\frac{2G(1+L/\rho)}
               {\sqrt{\mu_{\mathrm{PL}}\mu_{\mathrm{QG}}}}.
\]
If, additionally, $V_D(x_D^\star)=\{v_D^\star\}$ for both
compared datasets, then
\[
\|x_D^\star-x_{D^{(i)}}^\star\|+
\|v_D^\star-v_{D^{(i)}}^\star\|
\leq\frac{S_{\mathrm{sens}}}{n},\qquad
S_{\mathrm{sens}}:=2G\left(
\frac{(1+L/\rho)^2}{\sqrt{\mu_{\mathrm{PL}}\mu_{\mathrm{QG}}}}
+\frac1\rho\right).
\]
\end{lem}
\begin{proof}
At $u=x_{D^{(i)}}^\star$, Lemma~\ref{lem:phi-grad-diff} gives
$\|\nabla\Phi_D(u)\|\leq2G(1+L/\rho)/n$.
Combining outer PL with QG yields
\[
\frac{\mu_{\mathrm{QG}}}{2}\|u-x_D^\star\|^2
\leq\Phi_D(u)-\Phi_D^\star
\leq\frac{\|\nabla\Phi_D(u)\|^2}{2\mu_{\mathrm{PL}}},
\]
which proves the primal bound.  For the joint bound choose
$p\in V_D(u)$ with
$\|p-v_{D^{(i)}}^\star\|\leq2G/(\rho n)$, as in the proof of
Lemma~\ref{lem:phi-grad-diff}.  The singleton condition and
Lemma~\ref{lem:value-smooth}(i) imply
\[
\|v_D^\star-v_{D^{(i)}}^\star\|
\leq(L/\rho)\|x_D^\star-x_{D^{(i)}}^\star\|
             +2G/(\rho n).
\]
Add the primal bound.
\end{proof}

\begin{lem}[Coupled one-step bounds]\label{lem:coupled_corrected}
Set $\alpha=1$ and
$\Xi_t:=\Psi_{1,\mathcal D}(w_t,v_t)
       +\Psi_{1,\mathcal D^{(i)}}(w'_t,v'_t)$.
Under the geometric and sampling conditions of Lemma~\ref{lem:sgda_global_conv_theta} and any stepsizes $0<\eta_t\le\bar\eta$,
\begin{align}
 \E[\Psi_{1,D}(w_{t+1},v_{t+1})\mid\mathcal F_t]
 &\le(1-c\eta_t)\Psi_{1,D}(w_t,v_t)+C_{\mathrm{var}}'\eta_t^2,
 \label{eq:P-corr}\\
 \E[\Xi_{t+1}\mid\mathcal F_t]
 &\le(1-c\eta_t)\Xi_t+2C_{\mathrm{var}}'\eta_t^2.
 \label{eq:Xi-closed}
\end{align}
\end{lem}
\begin{proof}
Fix $D$ and suppress its subscript. Put
$g=\nabla\Phi(w_t)$, $p=\nabla_wF(w_t,v_t)$,
$q=\nabla_vF(w_t,v_t)$, and $\Delta=p-g$.
Smoothness of $\Phi$ and $F$, followed by conditional unbiasedness, gives
\begin{align*}
 \E[\Psi_{1}(w_{t+1},v_{t+1})\mid\mathcal F_t]-\Psi_{1}(w_t,v_t)
 &\le\eta_t(-2\langle g,p\rangle+\|p\|^2-\kappa_v\|q\|^2)\\
 &\quad+\eta_t^2\bigl(H\|p\|^2+(L\kappa_v^2/2)\|q\|^2
                         +C_{\mathrm{var}}'\bigr).
\end{align*}
The mixed term cancels:
$-2\langle g,p\rangle+\|p\|^2=-\|g\|^2+\|\Delta\|^2$.
The inner PL error bound and Lemma~\ref{lem:mismatch} imply
$\|\Delta\|\le(L/\rho)\|q\|$. Hence the first-order term is at most
$-\eta_t(\|g\|^2+\|q\|^2)$. Also
$\|p\|^2\le2\|g\|^2+2(L/\rho)^2\|q\|^2$.
The definition of $\bar\eta$ therefore bounds the mean-gradient Taylor
remainder by $\eta_t(\|g\|^2+\|q\|^2)/2$. Finally, the outer and inner
PL inequalities yield
\[
 -\tfrac12\|g\|^2-\tfrac12\|q\|^2
 \le-\mu_{\mathrm{PL}}(\Phi(w_t)-\Phi^\star)
      -\rho(\Phi(w_t)-F(w_t,v_t))
 \le-c\Psi_1(w_t,v_t).
\]
This proves~\eqref{eq:P-corr}; adding the two runs proves
\eqref{eq:Xi-closed}.
\end{proof}

\begin{proof}[Proof of Lemma~\ref{lem:sgda_global_conv_theta}]
By~\eqref{eq:P-corr}, writing $r=c_2+t$, induction gives
\[
 \E\Psi_{1,D}(w_{t+1},v_{t+1})
 \le\frac{d_1}{r}
     -\frac{cc_1d_1-C_{\mathrm{var}}'c_1^2}{r^2}
 \le\frac{d_1}{r}-\frac{d_1}{r^2}
 \le\frac{d_1}{r+1}.
\]
The almost-supermartingale theorem gives $\Psi_{1,D}(w_t,v_t)\to0$
almost surely because $\sum_t\eta_t=\infty$ and
$\sum_t\eta_t^2<\infty$. To establish parameter convergence, smoothness, anchored QG,
and Lemmas~\ref{lem:mismatch} and~\ref{lem:inner-error-bound} give
\begin{align*}
 \|(-\nabla_wF_D,\kappa_v\nabla_vF_D)(w_t,v_t)\|^2
 &\le C_H\Psi_{1,D}(w_t,v_t),\\
 C_H&=\max\{4L_\Phi^2/\mu_{\mathrm{QG}},(4+2\kappa_v^2)L^2/\rho\}.
\end{align*}
Thus the expected sum of the norms of the true-gradient increments is
finite, since $\sum_t\eta_t(c_2+t)^{-1/2}<\infty$. Their sum converges
absolutely almost surely. The noise increments form a martingale series
with summable second moments, so their sum also converges almost surely.
The iterates consequently have a finite limit. Continuity and
$\Psi_{1,D}\to0$ show that this limit is a minimax pair.
\end{proof}

\begin{lem}[One-step ascent under PL]\label{lem:one-step-ascent}
Let $g$ be $L$-smooth, attain $g^\star=\max_v g(v)$, and satisfy
$\frac12\|\nabla g(v)\|^2\geq\rho(g^\star-g(v))$.
For $v^+=v+\gamma\nabla g(v)$ and $0\leq\gamma\leq2/L$,
\[
g^\star-g(v^+)
\leq(1-2\rho\gamma+\rho L\gamma^2)(g^\star-g(v)).
\]
\end{lem}
\begin{proof}
Smoothness gives
$g(v^+)\geq g(v)+(\gamma-L\gamma^2/2)\|\nabla g(v)\|^2$.
The coefficient is nonnegative, so applying PL proves the claim.
\end{proof}

\begin{lem}[Closing the potential recursion]\label{lem:closeXi}
For either dataset $D$, define
\[
 B_T:=e^{-c\sum_{s=0}^{T-1}\eta_s}\Psi^{\max}_{1,0}
      +C_{\mathrm{var}}'\sum_{t=0}^{T-1}\eta_t^2
             e^{-c\sum_{s=t+1}^{T-1}\eta_s}.
\]
Then $\E\Psi_{1,D}(w_T,v_T)\le B_T$ and $\E\Xi_T\le2B_T$.
\end{lem}
\begin{proof}
Iterate~\eqref{eq:P-corr} and use $1-c\eta_s\le e^{-c\eta_s}$.
\end{proof}

\subsection{Proof of Main Result}\label{appendix: proof of main result}
\begin{movedblock}
\begin{prop}[On-average argument stability of SGDA for general stepsizes]\label{thm:sgda_stability_general}
Suppose that the loss function $f$ and datasets $\mathcal{D}, \mathcal{D}^{(i)}$ satisfy \ref{asm:smooth}--\ref{asm:unique} under the setting \ref{asm:stepsize}--\ref{asm:uniform}. Let $\varepsilon_T^w:=\frac1n\sum_{i=1}^n\E\|w_T(\mathcal D)-w_T(\mathcal D^{(i)})\|$. With $\alpha=1$ and the dual stepsize multiplier $\kappa_v$ above, the following on-average argument stability bound holds:\footnote{Under the Robbins--Monro conditions $\sum_t\eta_t=\infty$ and $\sum_t\eta_t^2<\infty$, the optimization term (the square root term) vanishes as $T\to\infty$  \citep{GarrigosGower2023Handbook}.}
\[
\varepsilon_T^w
\le
2\,C_{\mathrm{dist}}\,
\Bigl[
e^{-\,\tfrac{3c}{4}\sum_{s=0}^{T-1}\eta_s}\,\Psi^{\max}_{\alpha,0}
+
C_{\mathrm{var}}'\sum_{t=0}^{T-1}\eta_t^2\,e^{-\,\tfrac{3c}{4}\sum_{s=t+1}^{T-1}\eta_s}
\Bigr]^{1/2}
+\ \frac{C_{\mathrm{hit}}'}{n},
\]
where $\Psi^{\max}_{\alpha,0}
:=
\max\{
\E[\Psi_{\alpha,\mathcal D}(w_0,v_0)],\
\max_{1\le i\le n}\E[\Psi_{\alpha,\mathcal D^{(i)}}(w_0,v_0)]
\}$, $c:=\min\{\mu_{\mathrm{PL}}/2,\ \rho/2\}$, $C_{\mathrm{var}}'=(L_\Phi+L/2+L\kappa_v^2/2)\sigma_B^2$, and $C_{\mathrm{dist}}$, $C_{\mathrm{hit}}'$ are as in Theorem~\ref{thm:sgda_stability_noA7}.
If, in addition, $V_D(x_D^\star)=\{v_D^\star\}$ for every compared dataset, the same bound holds for
$\varepsilon_T^{\mathrm{joint}}:=n^{-1}\sum_i\E[\|w_T-w'_T\|+\|v_T-v'_T\|]$.
This additional condition is not needed for the primal stability or risk guarantees.
\end{prop}
\end{movedblock}
\begin{proof}[Proof of Proposition~\ref{thm:sgda_stability_general} and Theorem~\ref{thm:sgda_stability_noA7}]
Let $q_T$ denote the bracketed quantity in
Proposition~\ref{thm:sgda_stability_general}.  The optimization bound gives
$\E\Psi_{1,D}(w_T,v_T)\leq q_T$ for
$D\in\{\mathcal D,\mathcal D^{(i)}:i\in[n]\}$.
Insert the two empirical primal
minimizers between the coupled outputs.  Lemmas~\ref{lem:potential-to-distance}
and~\ref{lem:sensitivity} give
\[
\|w_T-w'_T\|
\leq\sqrt{\frac{2\Psi_{1,\mathcal D}(w_T,v_T)}{\mu_{\mathrm{QG}}}}
 +\sqrt{\frac{2\Psi_{1,\mathcal D^{(i)}}(w'_T,v'_T)}{\mu_{\mathrm{QG}}}}
 +\frac{S_w}{n}.
\]
Take expectations, apply Jensen's inequality, and average over $i$:
\[
\varepsilon_T^w
\leq2\sqrt{\frac{2q_T}{\mu_{\mathrm{QG}}}}+\frac{S_w}{n}.
\]
Since $C_{\mathrm{dist}}\geq\sqrt{2/\mu_{\mathrm{QG}}}$ and
$C_{\mathrm{hit}}'=S_{\mathrm{sens}}\geq S_w$, this proves the
displayed bound with the original constants.
Under the harmonic schedule with $cc_1>1$,
the sharper optimization estimate
$\E\Psi_{1,D}(w_T,v_T)\leq d_1/(c_2+T)$ gives
Theorem~\ref{thm:sgda_stability_noA7} by the same argument.
If $V_D(x_D^\star)=\{v_D^\star\}$ for every compared dataset,
the joint conclusions of Lemmas~\ref{lem:potential-to-distance}
and~\ref{lem:sensitivity} instead give
\[
\frac1n\sum_{i=1}^n\E\bigl[
\|w_T-w'_T\|+\|v_T-v'_T\|\bigr]
\leq2C_{\mathrm{dist}}\sqrt{q_T}
             +\frac{S_{\mathrm{sens}}}{n}.
\]
This proves the optional joint statement.
\end{proof}

\begin{proof}[Proof of Lemma~\ref{lem:gen}]
Fix the population pair $(w^\star,v^\star)$ from the statement. The envelope identity and interior optimality give
$\nabla_wF(w^\star,v^\star)=\nabla_vF(w^\star,v^\star)=0$.
The inner PL error bound supplies $u_D\in\arg\max_vF_D(w^\star,v)$ with
\[
\|u_D-v^\star\|\le\rho^{-1}\|\nabla_vF_D(w^\star,v^\star)\|.
\]
Consequently, the envelope identity and smoothness imply
\[
\|\nabla\Phi_D(w^\star)\|
\le\|\nabla_wF_D(w^\star,v^\star)\|
 +(L/\rho)\|\nabla_vF_D(w^\star,v^\star)\|.
\]
For independent data, each empirical block gradient at this fixed pair is an average of independent, mean-zero vectors of norm at most $G$. Its second moment is at most $G^2/n$. The $L^2$ triangle inequality therefore gives
$r_n\le(1+L/\rho)^2G^2/n$.

Write $x_D^\star$ for the empirical primal minimizer. Primal PL and the retained anchored QG condition yield
\[
\|x_D^\star-w^\star\|^2
\le\frac{\|\nabla\Phi_D(w^\star)\|^2}
{\mu_{\mathrm{PL}}\mu_{\mathrm{QG}}},\qquad
\E\|w_T-x_D^\star\|^2\le\frac{2a_T}{\mu_{\mathrm{QG}}}.
\]
By $L_\Phi$-smoothness of $R$ and $\nabla R(w^\star)=0$,
\begin{align}
\E[R(w_T)-R(w^\star)]
&\le\frac{L_\Phi}{2}\E\|w_T-w^\star\|^2\notag\\
&\le\frac{2L_\Phi}{\mu_{\mathrm{QG}}}a_T
 +\frac{L_\Phi r_n}{\mu_{\mathrm{PL}}\mu_{\mathrm{QG}}}.
\label{eq:PL-population-transfer}
\end{align}
For generalization, use
\[
\E[R(w_T)-\Phi_D(w_T)]
=\E[R(w_T)-R(w^\star)]-a_T
 +\E[R(w^\star)-\Phi_D^\star].
\]
Since $\E\Phi_D(w^\star)\ge R(w^\star)$ and primal PL gives
$\E[\Phi_D(w^\star)-\Phi_D^\star]\le r_n/(2\mu_{\mathrm{PL}})$,
the last term is at most $r_n/(2\mu_{\mathrm{PL}})$. Combining this with~\eqref{eq:PL-population-transfer} and dropping $-a_T$ proves the primal-risk claim. In particular, the empirical maximum need not be unbiased at a fixed $w^\star$; the displayed inequality accounts for its optimism.

For Lemma~\ref{lem:gen-pd}, use a common random seed in the compared runs and fix a deterministic comparator $v'$. Exchanging the independent observations $Z_i,Z_i'$ gives
\[
\E[F(w_T(D),v')-F_D(w_T(D),v')]
=\frac1n\sum_{i=1}^n\E\!\left[
 f(w_T(D^{(i)}),v';Z_i)-f(w_T(D),v';Z_i)\right].
\]
Its absolute value is at most
$G n^{-1}\sum_i\E\|w_T(D^{(i)})-w_T(D)\|$, uniformly in the fixed comparator. Taking maxima preserves this bound. The same argument for a fixed primal comparator bounds the difference of the two minima by
$G n^{-1}\sum_i\E\|v_T(D^{(i)})-v_T(D)\|$.
The triangle inequality gives $G\varepsilon_T^{\mathrm{joint}}$.
The gradient bounds must hold along these comparisons with the fixed comparator; a common bounded product domain is sufficient.
\end{proof}

\begin{proof}[Proofs of Theorems~\ref{thm:empRM_gen} and~\ref{thm:pop_excess_risk}]
The BRM primal risks equal $R$ and $\Phi_D$ in Lemma~\ref{lem:gen}.
Lemma~\ref{lem:sgda_global_conv_theta} gives $a_T\le d_1/(d_2+T)$.
Substituting this inequality and $r_n\le(1+L/\rho)^2G^2/n$ into Lemma~\ref{lem:gen} proves Theorem~\ref{thm:empRM_gen}. Substituting the same bounds into~\eqref{eq:PL-population-transfer} proves Theorem~\ref{thm:pop_excess_risk}.
\end{proof}

\paragraph{Stationary dependent data.}\label{app:stationary}
Let $g_{b,j}:=\nabla_bf(w^\star,v^\star;Z_j)$ for $b\in\{w,v\}$. The preceding fixed-pair argument requires only its second moment, and gives
$r_n\le(\sqrt{V_{w,n}}+(L/\rho)\sqrt{V_{v,n}})^2$,
where $V_{b,n}:=\E\|n^{-1}\sum_jg_{b,j}\|^2$.
For stationary observations with the target marginal distribution,
\[
V_{b,n}=\frac1n\left[
\E\|g_{b,0}\|^2
+2\sum_{k=1}^{n-1}(1-k/n)\E\langle g_{b,0},g_{b,k}\rangle
\right].
\]
Thus, if
\[
C_{\mathrm{cov}}:=1+\frac2{G^2}
\max_{b\in\{w,v\}}\sum_{k=1}^{\infty}
\bigl|\E\langle g_{b,0},g_{b,k}\rangle\bigr|<\infty,
\]
then $V_{b,n}\le G^2C_{\mathrm{cov}}/n$ and
$r_n\le(1+L/\rho)^2G^2C_{\mathrm{cov}}/n$.
The primal generalization and excess-risk proofs therefore apply with this dependence factor. Their target distribution must match the stationary marginal; arbitrary dependence or a different target requires the corresponding variance or distribution-shift term.

\subsection{Weak primal--dual risk}\label{app:weakPD}
\begin{movedblock}
Following \citet{Wang2022MCSGM}, the stability-to-generalization principle also controls the weak primal--dual gap whenever the joint output has on-average argument stability.

\begin{defn}[Weak primal--dual risk]\label{def:weakPD}
For a random algorithm output $(w_T,v_T)$, define the population and empirical weak--PD risks by
\begin{align*}
\Delta^{\mathrm{PD}}(w_T,v_T)
&:=\max_{v'\in\mathcal V}\E F(w_T,v')-\min_{w'\in\mathcal W}\E F(w',v_T),\\
\Delta^{\mathrm{PD}}_n(w_T,v_T)
&:=\max_{v'\in\mathcal V}\E F_{\mathcal D}(w_T,v')-\min_{w'\in\mathcal W}\E F_{\mathcal D}(w',v_T).
\end{align*}
The expectations include the data and algorithmic randomness.
\end{defn}

\begin{lem}[Weak primal--dual transfer]\label{lem:gen-pd}
Assume the conditions of Lemma~\ref{lem:gen} with independent observations. If the joint output has on-average argument stability
\[
\varepsilon_T^{\mathrm{joint}}
:=\frac1n\sum_{i=1}^n\E\bigl[
 \|w_T(\mathcal D)-w_T(\mathcal D^{(i)})\|
+\|v_T(\mathcal D)-v_T(\mathcal D^{(i)})\|\bigr],
\]
and the finite extrema in Definition~\ref{def:weakPD} can be taken over fixed bounded comparator sets containing the outputs, on whose product the block-gradient bounds hold, then
\[
\bigl|\Delta^{\mathrm{PD}}(w_T,v_T)
       -\Delta^{\mathrm{PD}}_n(w_T,v_T)\bigr|
\le G\varepsilon_T^{\mathrm{joint}}.
\]
The optional condition $\arg\max_vF_D(x_D^\star,v)=\{v_D^\star\}$ gives the joint stability bound through the potential-to-distance and saddle-sensitivity estimates; it is not needed for the primal-risk bound.
\end{lem}
In particular, under the conditions of Theorem~\ref{thm:empRM_gen} and the joint-stability conditions above, the output of Algorithm~\ref{alg:SGDA} satisfies
$|\Delta^{\mathrm{PD}}(\widehat w^{(T)},\widehat v^{(T)})-\Delta^{\mathrm{PD}}_n(\widehat w^{(T)},\widehat v^{(T)})|\le G\varepsilon_T^{\mathrm{joint}}$.
The proof of Lemma~\ref{lem:gen-pd} is part of the proof of Lemma~\ref{lem:gen} above.
\end{movedblock}

\section{Further discussions}

\subsection{Justification of Condition \ref{asm:grad}}\label{sec:A5proof}

For Algorithm~\ref{alg:SGDA}, Condition~\ref{asm:grad} is retained as an
assumption on the effective domain. The following lemma gives a sufficient
bounded-domain condition for the deterministic AGDA recursion stated in it.

\begin{lem}[Bounded effective domain and bounded per-sample gradients for AGDA]
\label{lem:A5_from_yang_AGDA}
Fix a dataset $\mathcal D=\{z_i\}_{i=1}^n$ and a per-sample objective $f(\cdot,\cdot;z)$.
Let the empirical saddle objective be
\[
F_{\mathcal D}(w,v) \;:=\; \frac{1}{n}\sum_{i=1}^n f(w,v;z_i),
\qquad
g_{\mathcal D}(w)\;:=\;\max_{v} F_{\mathcal D}(w,v),
\qquad
g_{\mathcal D}^\star := \min_{w} g_{\mathcal D}(w).
\]

Assume that $F_{\mathcal D}$ satisfies the following conditions stated using the
notation and constants of \citet{yang2020globalconvergencevariancereducedoptimization}:
there exist constants $l_{\mathrm{Y}}>0$ and $\mu_{1,\mathrm{Y}},\mu_{2,\mathrm{Y}}>0$ such that

\begin{enumerate}[leftmargin=2.2em]
\item[\textup{(Y1)}] (\textit{Lipschitz gradient}) $F_{\mathcal D}$ has $l_{\mathrm{Y}}$-Lipschitz gradients in the joint variable $(w,v)$
in the sense of \citet[Assumption~1]{yang2020globalconvergencevariancereducedoptimization}.
\item[\textup{(Y2)}] (\textit{Existence of saddle point}) $F_{\mathcal D}$ has at least one saddle point as in
\citet[Assumption~2]{yang2020globalconvergencevariancereducedoptimization}.
\item[\textup{(Y3)}] (\textit{Two-sided PL}) $F_{\mathcal D}$ satisfies the two-sided PL condition with constants
$\mu_{1,\mathrm{Y}},\mu_{2,\mathrm{Y}}$ as in \citet[Definition~2]{yang2020globalconvergencevariancereducedoptimization}.
\end{enumerate}

Run \emph{deterministic} AGDA (i.e., Algorithm~1 of \citet{yang2020globalconvergencevariancereducedoptimization}
with $\sigma^2=0$) with constant step sizes
\[
\tau_1 \;=\; \frac{\mu_{2,\mathrm{Y}}^{2}}{18\,l_{\mathrm{Y}}^{3}},
\qquad
\tau_2 \;=\; \frac{1}{l_{\mathrm{Y}}},
\]
initialized at $(w_0,v_0)$:
\[
w_{t+1}=w_t-\tau_1\nabla_w F_{\mathcal D}(w_t,v_t),
\qquad
v_{t+1}=v_t+\tau_2\nabla_v F_{\mathcal D}(w_{t+1},v_t).
\]
Define the potential (same form as in \citet[Eq.~(8)]{yang2020globalconvergencevariancereducedoptimization})
\[
P_t \;:=\; \bigl(g_{\mathcal D}(w_t)-g_{\mathcal D}^\star\bigr)
\;+\;\frac{1}{10}\,\bigl(g_{\mathcal D}(w_t)-F_{\mathcal D}(w_t,v_t)\bigr).
\]
Then there exist a saddle point $(w_{\mathcal D}^\star,v_{\mathcal D}^\star)$ of $F_{\mathcal D}$
and a constant $\alpha_{\mathrm{Y}}>0$ (depending only on $l_{\mathrm{Y}},\mu_{1,\mathrm{Y}},\mu_{2,\mathrm{Y}}$)
such that for all $t\ge 0$,
\[
\|w_t-w_{\mathcal D}^\star\|_2^2+\|v_t-v_{\mathcal D}^\star\|_2^2
\;\le\;
\alpha_{\mathrm{Y}}\,
\Bigl(1-\frac{\mu_{1,\mathrm{Y}}\mu_{2,\mathrm{Y}}^{2}}{36\,l_{\mathrm{Y}}^{3}}\Bigr)^{t}\,
P_0.
\]
In particular, all iterates remain in the closed Euclidean ball
\[
\Omega
\;:=\;
\Bigl\{(w,v): \|w-w_{\mathcal D}^\star\|_2^2+\|v-v_{\mathcal D}^\star\|_2^2
\le \alpha_{\mathrm{Y}}P_0\Bigr\},
\]
which is compact and convex.

Moreover, assume additionally that the sample space $\mathcal Z$ is compact (e.g., finite)
and that $(w,v,z)\mapsto \nabla_w f(w,v;z)$ and $(w,v,z)\mapsto \nabla_v f(w,v;z)$ are continuous.
Then the constant
\[
G
\;:=\;
\sup_{(w,v)\in\Omega,\ z\in\mathcal Z}
\max\Bigl\{\|\nabla_w f(w,v;z)\|_2,\ \|\nabla_v f(w,v;z)\|_2\Bigr\}
\]
is finite for these AGDA iterates. Condition~\ref{asm:grad} additionally requires the comparison points used in the stability and population-risk proofs to lie in the bounded domain.
\end{lem}

\begin{proof}

Under \textup{(Y1)}--\textup{(Y3)} and the specific choice of step sizes
$\tau_1=\mu_{2,\mathrm{Y}}^{2}/(18l_{\mathrm{Y}}^{3})$ and $\tau_2=1/l_{\mathrm{Y}}$,
Theorem~3.2 of \citet{yang2020globalconvergencevariancereducedoptimization} applies to the objective
$F_{\mathcal D}$ and to the deterministic AGDA recursion stated above.
Therefore, there exist a saddle point $(w_{\mathcal D}^\star,v_{\mathcal D}^\star)$ of $F_{\mathcal D}$
and a constant $\alpha_{\mathrm{Y}}>0$ (depending only on $l_{\mathrm{Y}},\mu_{1,\mathrm{Y}},\mu_{2,\mathrm{Y}}$)
such that for all $t\ge 0$,
\[
\|w_t-w_{\mathcal D}^\star\|_2^2+\|v_t-v_{\mathcal D}^\star\|_2^2
\;\le\;
\alpha_{\mathrm{Y}}\,
\Bigl(1-\frac{\mu_{1,\mathrm{Y}}\mu_{2,\mathrm{Y}}^{2}}{36\,l_{\mathrm{Y}}^{3}}\Bigr)^{t}\,
P_0.
\]

Since the contraction factor satisfies
\(
0 < 1-\frac{\mu_{1,\mathrm{Y}}\mu_{2,\mathrm{Y}}^{2}}{36\,l_{\mathrm{Y}}^{3}} < 1,
\)
we have
\(
\bigl(1-\frac{\mu_{1,\mathrm{Y}}\mu_{2,\mathrm{Y}}^{2}}{36\,l_{\mathrm{Y}}^{3}}\bigr)^t \le 1
\)
for all $t\ge 0$. Hence
\[
\|w_t-w_{\mathcal D}^\star\|_2^2+\|v_t-v_{\mathcal D}^\star\|_2^2
\;\le\;
\alpha_{\mathrm{Y}}P_0
\qquad\forall t\ge 0.
\]
This shows $(w_t,v_t)\in\Omega$ for all $t$.
The set $\Omega$ is a closed Euclidean ball in a finite-dimensional space, hence compact,
and it is convex by construction.

By assumption, the maps $(w,v,z)\mapsto \nabla_w f(w,v;z)$ and $(w,v,z)\mapsto \nabla_v f(w,v;z)$
are continuous, and $\Omega\times\mathcal Z$ is compact.
Therefore, by the extreme value theorem, the function
\[
(w,v,z)\mapsto
\max\Bigl\{\|\nabla_w f(w,v;z)\|_2,\ \|\nabla_v f(w,v;z)\|_2\Bigr\}
\]
attains its maximum on $\Omega\times\mathcal Z$, and the supremum
\[
G
=
\sup_{(w,v)\in\Omega,\ z\in\mathcal Z}
\max\Bigl\{\|\nabla_w f(w,v;z)\|_2,\ \|\nabla_v f(w,v;z)\|_2\Bigr\}
\]
is finite. This proves the stated gradient bound for the deterministic AGDA domain.
\end{proof}

\subsection{Neural network parametrization and PL/QG conditions}\label{appssec:Neural}

\paragraph{Width/radius and high-probability regimes for uniform constants.}

This subsection records the neural-network conditions relevant to
Conditions~\ref{asm:smooth}--\ref{asm:grad}, particularly the PL
conditions for the primal value function and the auxiliary regression. Fix a dataset
$D=\{z_i\}_{i=1}^{n}$ and write, as in Section~\ref{sec:sgda},
\[
F_D(w,v):=\frac1n\sum_{i=1}^n f(w,v;z_i),\qquad
\Phi_D(w):=\max_v F_D(w,v),\qquad
\Phi_D^\star:=\min_w\Phi_D(w).
\]
Here $w$ parametrizes the action-value network $Q_w$ and $v$ parametrizes the
auxiliary function $\zeta_v$ in the bi-conjugate BRM objective
\eqref{eq:relaxed-per-sample}.  The PL/QG statements below are local:
fix a failure level $\delta\in(0,1)$ and a radius $R>0$, and take the effective
domain in Condition~\ref{asm:grad} to be the initialization neighborhood
\[
\Omega_R
:=
\{(w,v):\|(w,v)-(w_0,v_0)\|\le R\}\cap(\mathcal W\times\mathcal V).
\]
All constants used in the stability theorem are then understood on this same
set: $L$ is the joint smoothness constant of $F_D$, $\rho$ is the inner
PL constant in $v$, $G$ is the per-sample gradient bound, and
$\mu_{\mathrm{PL}},\mu_{\mathrm{QG}}$ are the PL and QG constants for
$\Phi_D$.

Consider a depth $(H+1)$ fully connected network of width $m$, initialized with
i.i.d.\ Gaussian weights.  Let $\lambda_0>0$ denote the minimum eigenvalue of
the empirical tangent kernel at initialization on the $n$ sample inputs, and
let $\rho_\sigma>0$ be the lower bound on the derivative of the last-layer
activation used in \citet{liu2022loss}.  Choose the target PL constant
$\mu_{\mathrm{PL}}$ so that
\[
0<\mu_{\mathrm{PL}}<\rho_\sigma^2\lambda_0 .
\]
Combining the over-parameterized neural-network PL result of
\citet{liu2022loss} with the BRM reduction of
\citet{kang2025empiricalriskminimizationapproach}, there are
problem-dependent constants, with logarithmic factors hidden in
$\widetilde{\Omega}(\cdot)$, such that, if
\[
m=\widetilde{\Omega}\!\left(
\frac{n\,R^{6H+2}}
{\bigl(\lambda_0-\mu_{\mathrm{PL}}/\rho_\sigma^{2}\bigr)^2}
\right),
\]
then with probability at least $1-\delta$ the value function $\Phi_D$ satisfies
the PL inequality throughout the primal projection of $\Omega_R$:
\[
\frac12\|\nabla\Phi_D(w)\|^2
\ge
\mu_{\mathrm{PL}}\bigl(\Phi_D(w)-\Phi_D^\star\bigr).
\]
Thus the constant appearing in Condition~\ref{asm:PL-Phi} and in the
contraction factor $c=\min\{\mu_{\mathrm{PL}}/2,\rho/2\}$ in
Theorem~\ref{thm:sgda_stability_noA7} can be taken to be any admissible
$\mu_{\mathrm{PL}}\in(0,\rho_\sigma^2\lambda_0)$ for which the displayed width
condition is imposed.

The local geometric statements hold on the intersection of this PL event
with the event that the
SGDA iterates and the auxiliary ascent trajectories used in the analysis
remain inside $\Omega_R$; see Theorem~7 of \citet{liu2022loss} for the
corresponding radius/width regime for stochastic gradient methods.
The optimization recursion uses stochastic-gradient moments in the
original pre-update filtration, without conditioning on future
confinement, and unconditional risk bounds require control of the
complementary event.
On this event, the softmax Bellman operator is smooth because $\eta>0$.
For the auxiliary network, the squared regression in
\eqref{eq:relaxed-per-sample} gives an inner PL condition when its
residual Jacobian is uniformly nondegenerate, as follows.

Group repeated state--action inputs and let $r_{D,w}(v)$ collect the
differences between $\zeta_v$ and the empirical conditional means of
$V_{Q_w}(s')$, weighted by the square roots of their empirical frequencies.  Suppose these conditional means
are realizable and write $J_{D,v}:=\nabla_v r_{D,w}(v)$.
If $J_{D,v}J_{D,v}^{\mathsf T}\succeq\kappa_{\mathrm{aux}} I$ throughout the
effective domain, then
\[
\Gamma_D(w,v)=\beta^2\|r_{D,w}(v)\|^2,\qquad
\frac12\|\nabla_vF_D(w,v)\|^2
=2\beta^4\|J_{D,v}^{\mathsf T}r_{D,w}(v)\|^2
\geq2\beta^2\kappa_{\mathrm{aux}}\Gamma_D(w,v).
\]
Thus Condition~\ref{asm:SC} holds with $\rho=2\beta^2\kappa_{\mathrm{aux}}$.
The lower bound is a condition on the auxiliary network's tangent
kernel, alongside the primal kernel condition above
\citep{liu2022loss,kang2025empiricalriskminimizationapproach}.
It permits redundant parameters and does not require concavity in $v$.
For the population objective, the same argument applies to the
conditional-mean regression when the corresponding residual-space
Jacobian lower bound and realizability hold.

The outer PL condition yields quadratic growth relative to
$\arg\min_w\Phi_D(w)$, by the same argument as
Lemma~\ref{lem:inner-error-bound}.  The parameter-distance results
use Condition~\ref{asm:QG-Phi} as stated: on the effective domain,
the primal minimizer is identified, so this set is $\{x_D^\star\}$
and the growth bound takes the form
\[
\Phi_D(w)-\Phi_D^\star
\geq\frac{\mu_{\mathrm{QG}}}{2}\|w-x_D^\star\|^2.
\]
This identification condition is imposed separately from network width.

For a neighboring dataset $D^{(i)}$, the same
argument is applied to $F_{D^{(i)}}$ and $\Phi_{D^{(i)}}$; applying a union
bound and then taking the larger upper-bound constants for $L,G$ and the
smaller lower-bound constants for $\rho,\mu_{\mathrm{PL}},\mu_{\mathrm{QG}}$
gives the uniform constants required by Condition~\ref{asm:uniform}.

For the remaining conditions, joint smoothness follows under the
regularity conditions in Lemma~29 of
\citet{kang2025empiricalriskminimizationapproach}, and the bounded
gradient condition is discussed in Appendix~\ref{sec:A5proof}.
The primal stability and risk guarantees allow multiple dual
maximizers.  The joint parameter bound in
Theorem~\ref{thm:sgda_stability_noA7} additionally requires a singleton
dual maximizer at the empirical primal optimum; this property is
checked separately when that bound is used.

\subsection{How index pairing works}\label{appssec:index-pair}

By synchronizing the minibatch selection, we use the same algorithmic randomness in the two runs. The proof compares each output with its empirical primal minimizer and bounds the change in that minimizer after one observation is replaced. Fresh uniform minibatches give the conditional optimization recursion even when the offline observations are dependent. This conditional stability statement is distinct from the population-risk calculation, where dependence is measured by the fixed-point gradient covariance.

\subsection{Uniformity setting}\label{appssec:uniformity}

The setting \ref{asm:uniform} is standard in the stability literature: for example, \citet{hardt2016train} assumed each per-example loss $f(\cdot ; z)$ is $L$-Lipschitz and $\beta$-smooth uniformly in $z$, and \citet{Wang2022MCSGM} assumed the gradients and smoothness of $f(w, v ; z)$ are bounded by global constants $G$ and $L$ for all $z$. These conditions give common smoothness and gradient bounds across datasets; the PL, QG, and initial-potential constants must also be controlled as required by Setting~\ref{asm:uniform}.

\subsection{Discussion of contributions}\label{appssec:discussion}

\paragraph{What \citet{kang2025empiricalriskminimizationapproach} give us, and what they do not.} \citet{kang2025empiricalriskminimizationapproach} establish that, after a bi-conjugate transformation, the population and empirical BRM objectives admit PL structure under their function-class conditions. They analyze global convergence to the empirical optimum under their auxiliary-block assumptions. However, they:
\begin{itemize}
    \item do not analyze algorithmic stability,
    \item do not analyze generalization (population vs empirical BRM). In fact, Lemma~28 in the latest version of \citet{kang2025empiricalriskminimizationapproach} uses our result for the BRM sample complexity: we use their PL structure for BRM, and they use our sample-complexity result.
\end{itemize}

\paragraph{What \citet{Wang2022MCSGM} give us, and what they do not.}
\citet{Wang2022MCSGM} develop on-average argument stability for Markov chain SGD/SGDA and show that \edit{for smooth convex--concave losses, the weak primal--dual population gap in their Theorem~9 has rate $\widetilde O(n^{-1/2})$. Their Theorem~10 gives this rate for excess primal population risk with the additional assumption of population inner strong concavity}, with explicit dependence on the mixing parameter of the Markov chain over indices. We summarize the difference between \citet{Wang2022MCSGM} and our paper in Table \ref{tab:comparison}. We reuse from \citet{Wang2022MCSGM}:
\begin{itemize}
    \item The concept of on-average argument stability and its connection to weak primal--dual generalization, used in \del{Lemma~\ref{lem:gen}}\edit{Remark~\ref{rem:weakPD} and Lemma~\ref{lem:gen-pd}}.
    \item The two-run coupling on neighboring datasets. Our proof then uses empirical-minimizer sensitivity and the PL potential.
\end{itemize}
However, there are two crucial structural differences:
\begin{itemize}
    \item Geometry: \edit{\citet{Wang2022MCSGM} use convex--concave geometry for weak primal--dual risk and additionally require strong concavity of the population inner objective for their excess primal-risk theorem.} Our BRM objective is nonconvex in the primal parameters; only the value function $\Phi_D(w):=\max _v F_D(w, v)$ is PL (and QG), while the auxiliary maximization satisfies PL in $v$. The algorithm, however, is run on the bilevel saddle objective $F_D(w, v)$, not on $\Phi_D$. \edit{The convexity-based coupled-distance argument used by \citet{Wang2022MCSGM} therefore does not directly supply a stability bound for this BRM geometry.}
    \item Object of interest: \citet{Wang2022MCSGM} work directly with the original minimax $F$ and directly analyze its risk/generalization. Our primary object is the value function $\Phi_D$ (Bellman residual); the weak primal--dual gap is an additional consequence when joint stability holds. The primal gradients used by SGDA at time $t$ can differ from $\nabla\Phi_D(w_t)$ when the dual variable is not at an inner maximizer: $$
\Delta_t:=\nabla_w F_D\left(w_t, v_t\right)-\nabla \Phi_D\left(w_t\right)
$$
Controlling this mismatch under the PL geometry is precisely what drives the need for our Lyapunov potential.

\end{itemize}

\begin{table}[ht!]
    \centering
    \small 
    \caption{\edit{Comparison with prior work. Our statistical rate uses independent observations or a controlled stationary covariance factor. Markov update indices and dependent training observations are distinct settings. Table~\ref{tab:comparison-details} gives the detailed assumptions and metrics.}}
    \label{tab:comparison}
    \begin{tabularx}{\textwidth}{@{}l X l l l@{}}
        \toprule
        \textbf{Work} & \textbf{Geometry} & \textbf{Setting} & \textbf{Stat. Rate} & \textbf{Opt. Rate}  \\
        \midrule
        This paper 
            & Primal/dual PL 
            & Offline BRM 
            & $\mathcal{O}(1/n)$ 
            & $\mathcal{O}(1/T)$ 
            \\ 
        \addlinespace
        \citet{Wang2022MCSGM}
            & Convex-concave 
            & minimax 
            & $\mathcal{O}(1/\sqrt{n})$ 
            & $\mathcal{O}(1/\sqrt{T})$ 
            \\
        \bottomrule
    \end{tabularx}
\end{table}

\begin{table}[htbp]
\centering
\footnotesize
\setlength{\tabcolsep}{2.5pt}
\caption{Detailed comparison of the statistical bounds in Table~\ref{tab:related-compact}. The notes below the table explain the risk measures, the rates, and the data assumptions.}
\label{tab:comparison-details}
\begin{tabularx}{\linewidth}{@{}lcl>{\raggedright\arraybackslash}p{0.11\linewidth}>{\raggedright\arraybackslash}p{0.12\linewidth}>{\raggedright\arraybackslash}X>{\raggedright\arraybackslash}p{0.15\linewidth}l@{}}
\toprule
Work & BRM & Class & Estimator & Data & Assumptions & Risk & Rate \\
\midrule
\citet{maillard2010finite}
 & Yes & linear & Oracle for empirical BRM
 & Generative model, two next states per query
 & --
 & Excess $\overline{\mathcal L_{\mathrm{BE}}}$ for $\mathcal T^\pi$ (Thm.~7)
 & $O(n^{-1/2})$ \\
\addlinespace
\citet{dai2018sbeed}
 & Yes & general & Oracle for empirical BRM
 & Single $\beta$-mixing path
 & --
 & Excess $\mathcal R$-type objective (Thm.~6)
 & $\widetilde O(n^{-1/2})$ \\
\addlinespace
\citet{chen2019information}
 & Yes & finite & Oracle for empirical BRM
 & Independent transitions
 & Realizability, completeness
 & $\overline{\mathcal L_{\mathrm{BE}}}(\widehat Q)$ (proof of Thm.~3)
 & $\widetilde O(n^{-1})$ \\
\addlinespace
\citet{duan2021risk}
 & Yes & general & Oracle for empirical BRM or FQI
 & Independent transitions
 & Completeness
 & Excess $\overline{\mathcal L_{\mathrm{BE}}}$
 & $O(n^{-1/2})$ \\
\addlinespace
\citet{hu2025fast}
 & Yes & linear & Oracle for empirical BRM
 & Independent transitions
 & Realizability, modified completeness, feature coverage
 & $\overline{\mathcal L_{\mathrm{BE}}}(\widehat Q)$ via $\|\widehat Q-Q^\star\|_\infty$
 & $\widetilde O(n^{-1})$ \\
\addlinespace
\citet{Wang2022MCSGM}
 & No & general & Averaged MC-SGDA
 & Independent observations, Markov update indices
 & Convex--concave (Thm.~10: population inner strong concavity)
 & $\Delta^{\mathrm{PD}}$ (Thm.~9); $R(w)-R(w^\star)$ (Thm.~10)
 & $\widetilde O(n^{-1/2})$ \\
\midrule
\textbf{Ours}
 & Yes & general & Last iterate of SGDA, $T\asymp n$
 & Independent transitions or one stationary path
 & Primal and inner PL, QG, \ref{asm:smooth}--\ref{asm:unique}
 & $\mathcal R(\widehat w^{(T)})-\mathcal R(w^\star)$
 & $O(n^{-1})$ \\
\bottomrule
\end{tabularx}

\medskip
\begin{minipage}{\linewidth}
\footnotesize
\textbf{Risk.} $\overline{\mathcal L_{\mathrm{BE}}}(\widehat Q)$ is the MSBE of Section~\ref{sec:setup} at the learned $\widehat Q$, with the soft operator $\mathcal T$ replaced by the hard-max operator for \citet{chen2019information} and \citet{hu2025fast} and by the policy-evaluation operator $\mathcal T^\pi$ for \citet{maillard2010finite}. ``Excess'' subtracts the best value in the hypothesis class, $\min_{Q\in\mathcal F}\overline{\mathcal L_{\mathrm{BE}}}(Q)$, up to approximation terms. The objective of \citet{dai2018sbeed} is a smoothed bi-conjugate Bellman error, the analogue of $\mathcal R$ in Theorem~\ref{thm:empRM_gen}. $\Delta^{\mathrm{PD}}$ is the weak primal--dual gap (Definition~\ref{def:weakPD}) and $R$ the primal risk (Definition~\ref{def:primalRisks}); for our BRM objective, $\mathcal R(\widehat w^{(T)})-\mathcal R(w^\star)$ equals the excess MSBE under auxiliary-function realizability.

\smallskip
\textbf{Rates.} Rates show the dependence on $n$ for fixed problem parameters, and $\widetilde O$ hides logarithmic factors. For \citet{chen2019information} and \citet{hu2025fast}, the headline guarantees concern the value gap or regret, and the table reports the squared-loss bound behind them. For \citet{hu2025fast}, Bellman contraction and realizability give $\|Q-\mathcal TQ\|_\infty\le(1+\beta)\|Q-Q^\star\|_\infty$, so their sup-norm bound controls $\overline{\mathcal L_{\mathrm{BE}}}$; their regret result additionally uses a margin condition. The $O(n^{-1/2})$ rate of \citet{duan2021risk} holds for general classes, and their local Rademacher analysis gives $O(n^{-1})$ for finite or linear classes when the approximation terms vanish. Our bound is $O(1/n+1/T)$ (Theorem~\ref{thm:pop_excess_risk}), so $T\asymp n$ iterations give $O(n^{-1})$; a uniform auxiliary approximation error $\epsilon_{\mathrm{aux}}$ adds $\beta^2\epsilon_{\mathrm{aux}}$, as in~\eqref{eq:paraExpRM}.

\smallskip
\textbf{Data.} From the strongest to the weakest requirement: a generative model that returns two independent next states for the same query \citep{maillard2010finite} removes the double-sampling problem by assumption and cannot be built from logged data. Independent transitions, each drawn from a fixed state--action distribution followed by $P$, are the standard idealization of an offline dataset \citep{chen2019information,duan2021risk,hu2025fast,Wang2022MCSGM}; in \citet{Wang2022MCSGM} the Markov chain drives the update indices, while the observations themselves are independent. A single stationary trajectory is what a behavior policy produces in practice. \citet{dai2018sbeed} assume exponential $\beta$-mixing for it, whereas our population bounds need only summable gradient covariances at the population optimum ($C_{\mathrm{cov}}<\infty$, Appendix~\ref{app:stationary}), which exponential $\beta$-mixing implies for bounded gradients; our stability bound (Theorem~\ref{thm:sgda_stability_noA7}) needs no condition on the observations.
\end{minipage}
\end{table}

\subsection{Discussion of potential extensions} \label{appssec:limitation}


\paragraph{On non-smooth BRM.}
Our analysis, which is based on \emph{smooth} Bellman--residual objectives, does not lose any generality for the entropy-regularized formulation adopted in Section~\ref{sec:setup}. In particular, when $\eta>0$, the optimality equations involve the entropic log-sum-exp (softmax) operator
\[
V_Q(s)=\eta\log\sum_{a\in\mathcal A}\exp(Q(s,a)/\eta),
\]
so the Bellman operator and the resulting residual objective are differentiable and (under standard boundedness conditions on the function class) admit Lipschitz gradients. By contrast, genuinely non-smooth BRM variants (e.g., the hard $\max$ Bellman operator obtained in the zero-temperature limit $\eta\to 0$, or objectives with non-differentiable losses/parameterizations) require different tools, such as proximal or sub-gradient dynamics, smoothing/Moreau-envelope arguments, and stability analyses tailored to non-smooth saddle problems, and are therefore orthogonal to the paper's goal.

\paragraph{From Bellman residuals to policy performance.}

Our main theorem is stated for the Bellman Residual Minimization (BRM) objective. In offline RL, such a result is often converted to a control-performance guarantee that is standard once one imposes a coverage (a.k.a. concentrability) assumption that relates the state--action occupancy of the target/learned policy to the data distribution. Under such coverage, a bound on the population Bellman residual under the dataset distribution immediately yields a bound on the Bellman error along the policy's visitation distribution, which then translates to policy suboptimality via standard approximate dynamic programming arguments. In fact, \citet{kang2025empiricalriskminimizationapproach} uses our result to instantiate this pipeline for offline inverse reinforcement learning and related entropy-regularized control problems.

\section{Related Works}\label{appssec:related works}
    
\paragraph{Bellman Residual Minimization (BRM).}
Bellman Residual Minimization (BRM) can be viewed as a regression-style approach to offline reinforcement learning, in which the squared Bellman residual plays the role of a regression loss. It has been identified as a particularly direct way to enforce Bellman consistency under flexible function approximation, including nonparametric or neural-network parameterizations \citep{jiang2024offline}. A central obstacle, however, is the double-sampling problem \citep{antos2008learning}: the squared Bellman residual involves the square of a conditional expectation, which cannot be unbiasedly estimated from a single next-state sample. An influential remedy is to introduce a dual, or debiasing, correction through a bi-conjugate reformulation \citep{antos2008learning,dai2018sbeed,patterson2022generalized}. While this removes the statistical bias caused by double sampling, it transforms BRM into a minimax optimization problem, whose primal component is generally nonconvex under nonlinear function approximation. Consequently, BRM-style methods have long been viewed as computationally delicate in offline RL, despite their conceptual appeal. 

Recently, \citet{kang2025empiricalriskminimizationapproach} showed that, for entropy-regularized BRM and under neural-network approximation satisfying suitable over-parameterization conditions, the bi-conjugate objective admits PL geometry and gradient-based minimax optimization converges globally under their auxiliary-block conditions. However, their result is primarily an optimization guarantee: it does not directly establish algorithmic stability or a statistical generalization bound comparing the empirical BRM objective with its population counterpart. This statistical question is the main focus of the present work.

\paragraph{Algorithmic stability and generalization.}
Algorithmic stability is a classical approach to generalization analysis that controls how sensitively the output of a learning algorithm responds to small perturbations of the training dataset \citep{bousquet2002stability}. This framework has been developed through several notions of stability, including uniform stability, hypothesis stability, argument stability, and on-average stability \citep{bousquet2002stability,hardt2016train,kuzborskij2018data,lei2020fine}. For gradient-based learning, \citet{hardt2016train} established influential stability bounds for stochastic gradient descent, and subsequent works refined these bounds using data-dependent or on-average notions of stability \citep{kuzborskij2018data,lei2020fine,Lei2021SGDAStability}. Stability has also been used to study algorithms that converge to global optima \citep{charles2018stability} and to obtain refined high-probability generalization guarantees for uniformly stable algorithms \citep{feldman2018generalization,bousquet2020sharper}.

More recently, stability-based generalization has been extended from minimization to minimax optimization. This direction is particularly relevant for BRM because the bi-conjugate correction turns the squared Bellman-residual objective into a saddle-point problem. Prior works studied the stability and generalization of SGDA under various geometries, with recent extensions to decentralized, distributed, and differentially private SGDA variants
\citep{farnia2021train,Lei2021SGDAStability,zhu2023stability,zhang2024generalization}. In parallel, \citet{Wang2022MCSGM} developed an on-average argument-stability framework for stochastic gradient methods with Markov-chain sampling and established stability-to-generalization guarantees for both minimization and minimax problems. Their Markov chain samples update indices from an independent training dataset; dependence among trajectory observations requires a separate statistical argument. However, these existing minimax stability results do not directly apply to the BRM problem studied here. The key difficulty is geometric: the BRM saddle objective is not convex in the primal parameter, and the algorithm is run on $F_D(w,v)$ rather than directly on the primal value function $\Phi_D(w):=\max_v F_D(w,v)$. Thus the primal update contains a mismatch term
\[
\nabla_w F_D(w_t,v_t)-\nabla \Phi_D(w_t),
\]
unless the dual variable is already maximized. Our analysis addresses this mismatch through a Lyapunov potential that couples the PL suboptimality of $\Phi_D$ with the dual gap, yielding an $O(1/n)$ primal argument-stability term and, through the same geometry, the BRM generalization guarantee.

\paragraph{PL geometry and over-parameterized neural networks.}
The Polyak--{\L}ojasiewicz (PL) condition is a central relaxation of strong convexity that still yields global convergence of gradient-based methods despite nonconvexity \citep{polyak1963gradient,Karimi2016PL}. It has become particularly important in the analysis of over-parameterized models, where classical convexity-based arguments are often inappropriate. In modern neural-network theory, over-parameterized loss landscapes have been studied through several complementary perspectives, including the absence of strict spurious local minima, connectivity of sublevel sets or global-minimum manifolds, neural tangent kernel (NTK) dynamics, and local or global PL-type conditions
\citep{du2018gradient,jacot2018neural,lee2019wide,allen2019convergence,arora2019fine,soltanolkotabi2018theoretical,liu2022loss,chen2023overparameterized, xu2025local}. In particular, \citet{liu2022loss} argue that sufficiently over-parameterized nonlinear systems, including wide neural networks, satisfy a PL$^\ast$-type condition in a neighborhood of the initialization, and relate this property to tangent-kernel conditioning and Hessian norm control. This line of work explains why gradient-based methods can converge globally in highly nonconvex models without relying on convexity. 

Our paper uses this perspective in the BRM setting: the relevant population and empirical soft Bellman-residual objectives are analyzed under primal and dual PL conditions, with the additional QG and domain requirements stated explicitly; see \citet{kang2025empiricalriskminimizationapproach} for the BRM geometry. We then complement this optimization geometry with an algorithmic-stability analysis to obtain finite-sample statistical guarantees.

\paragraph{Projected Mean-Squared Bellman Error (MSPBE)}
Another approach is to minimize the projected mean-squared Bellman error (MSPBE), which first projects the Bellman residual onto the approximation space and then measures its squared norm \citep{patterson2022generalized}. This differs from BRM, which directly minimizes the mean-squared Bellman error (MSBE) and therefore penalizes the full Bellman inconsistency. MSPBE instead evaluates only the component of the Bellman residual that lies within the representable function space. From an optimization perspective, this projection can make the objective more tractable, since it only requires the residual to be small along directions captured by the function class. However, this advantage comes at a cost: MSPBE ignores the component of the Bellman residual orthogonal to the approximation space, and therefore does not control the full Bellman inconsistency.

\paragraph{Fitted Q-Iteration (FQI)}
Fitted Q-Iteration (FQI), introduced by \citet{ernst2005tree}, is a canonical value-based method for offline reinforcement learning. Although conceptually simple, iterative value-based procedures are known to suffer from instability. In particular, approximate value iteration / fitted value iteration can be unstable under off-policy sampling and function approximation, including the case of infinite data and exact regression with linear function approximation. This instability arises because FQI repeatedly fits value functions to targets that themselves depend on the current estimate. Thus, the regression target evolves over iterations, creating an inherently unstable learning process. More broadly, this phenomenon reflects the well-known deadly triad: the combination of function approximation, bootstrapping, and off-policy learning \citep{tsitsiklis1996feature,van2018deep,sutton2018reinforcement}.

\paragraph{Marginalized Importance Sampling (MIS)}
Marginalized importance sampling (MIS) methods cast offline reinforcement learning as a minimax problem by learning marginalized density ratios, often interpreted as discriminators, to control worst-case reweighted Bellman errors \citep{nachum2019dualdice,zhang2020gendice,liu2018breaking}. Unlike BRM, these methods optimize linear Bellman-error objectives rather than squared residuals, and therefore avoid the double-sampling problem. However, MIS objectives typically involve signed average Bellman errors. As a result, positive and negative residuals can cancel, so these objectives may fail to penalize pointwise Bellman inconsistency \citep{nachum2019dualdice,zhang2020gendice}. This distinction motivates our focus on BRM as a direct residual-minimization approach, while addressing its statistical and optimization challenges through the bi-conjugate formulation and stability analysis.

\begin{editblock}
\section{Details of the numerical illustration}\label{app:experiments}

This appendix gives the full setting and additional results for Section~\ref{sec:exp}. Code and raw
results are available at \url{https://github.com/jajajang/soft-brm-stability}.

\paragraph{Setup.}
The MDP has $|\mathcal S|=5$, $|\mathcal A|=3$, $\beta=0.9$ and $\eta=1$. Rewards are drawn from
$\mathrm{Unif}[0,1]$, and $P(\cdot\mid s,a)=\tfrac12\delta_s+\tfrac12 p_{s,a}$ with
$p_{s,a}\sim\mathrm{Dirichlet}(\mathbf 1)$ (fixed seed). The behavior policy is uniform, and $\mu$
denotes its stationary state--action distribution (smallest mass $0.046$). We use two data regimes:
(i) $n$ i.i.d.\ transitions with $(s,a)\sim\mu$ and $s'\sim P(\cdot\mid s,a)$; (ii) a single stationary
trajectory of length $n$ under the behavior policy, whose state chain has second eigenvalue of modulus
$0.59$, so consecutive transitions are dependent. The neighbor $\mathcal D^{(i)}$ replaces a uniformly
chosen $z_i$ by an independent copy $\tilde z_i\sim\mu\otimes P$.

$Q_w$ maps a one-hot state to $|\mathcal A|$ outputs through $H$ softplus units, and $\zeta_v$ maps a
one-hot state--action pair to a scalar through $H$ softplus units, with $H\in\{8,16,32\}$. Both use the
NTK parameterization (output layer scaled by $H^{-1/2}$) with standard Gaussian initialization, and the
output biases start at the constant value scale. Softplus keeps $F_{\mathcal D}$ smooth, as required by
\ref{asm:smooth}. We run Algorithm~\ref{alg:SGDA} with minibatch size $B=32$, stepsizes
$\eta_t=400/(2000+t)$, dual multiplier $\kappa_v=2$ and $T=40{,}000$ iterations, from the same
initialization and with the same minibatch index sequence on $\mathcal D$ and $\mathcal D^{(i)}$. Each
$(n,H)$ pair uses $50$ repetitions, each with a fresh dataset, initialization, index sequence and replaced
index, for $n\in\{100,200,\dots,6400\}$. All runs use NumPy on CPU, and the full study (both data regimes and the constant estimates below) takes about one hour on two CPU cores.

We report three quantities: the primal argument stability $\varepsilon^w_T$ of
Theorem~\ref{thm:sgda_stability_noA7}; the function-space distance
$\|Q_{w_T}-Q_{w'_T}\|_{\mu}$; and the excess population MSBE of the output on $\mathcal D$, computed
exactly from $P$. With one-hot inputs both networks can represent any function of their inputs, so
$Q^\star$ is representable and the auxiliary class contains the conditional means; so $\mathcal R(w^\star)=0$ and the excess MSBE equals the left-hand side of
Theorem~\ref{thm:pop_excess_risk}. As a reference we also compute the population MSBE of the exact
empirical minimizer, the fixed point of the empirical soft Bellman operator. Slopes are least-squares fits
of $\log(\text{mean})$ on $\log n$ over $n\ge 200$, with $95\%$ bootstrap intervals over repetitions.

\paragraph{Results.}
Figure~\ref{fig:exp} in Section~\ref{sec:exp} and Tables~\ref{tab:exp-stability} and~\ref{tab:exp-msbe} report the results. Both
quantities decay as $n^{-1}$: the fitted slopes are between $-1.08$ and $-1.05$ for
$\varepsilon^w_T$ and between $-1.01$ and $-0.97$ for the excess MSBE, every $95\%$ interval
excludes $-1/2$, and $n\,\varepsilon^w_T$ and $n\cdot\mathrm{MSBE}$ stay roughly constant over a $32$-fold
range of $n$. The same holds for a single dependent trajectory (slopes
$-1.01$ and $-0.99$), in line with the fact that
the stability bound needs no independence of the observations. The generalization gap
$\mathcal R(w_T)-\widehat{\mathcal R}_n(w_T)$ of Theorem~\ref{thm:empRM_gen} behaves in the same way
($n$ times the gap lies in $[0.34,0.45]$ for $n\ge200$). No width shows saturation; the
narrowest network has a larger constant but the same rate.

The optimization error is small at $T=40{,}000$: $\varepsilon^w_T$ changes by at most
$15\%$ between $T=20{,}000$ and $T=40{,}000$, and the excess MSBE of
the SGDA output is within $10\%$ of that of the exact empirical minimizer for
$200\le n\le3200$ (within $33\%$ at $n=6400$, where residual optimization error
becomes visible). At $n=100$ some datasets miss a state--action pair, which inflates the error; this point
is excluded from the fits.

\paragraph{Magnitude of the constants.}
For $H=16$ and $n=800$ we estimated the constants along five SGDA trajectories on their own datasets:
$\mu_{\mathrm{PL}}$ and $\rho$ as the smallest observed ratios in \ref{asm:PL-Phi} and \ref{asm:SC},
$\mu_{\mathrm{QG}}$ from \ref{asm:QG-Phi} relative to the final iterate of a run four times longer
(a proxy for the empirical minimizer), $G$ as the largest per-sample gradient norm, and $L$ by power
iteration on the joint Hessian of $F_{\mathcal D}$. Over the whole trajectories the worst values are
$\mu_{\mathrm{PL}}=0.018$, $\rho=0.014$, $\mu_{\mathrm{QG}}=0.012$,
$G=2.5$ and $L=4.2$. With these, the constant $S_w$ of Lemma~\ref{lem:sensitivity}, which multiplies $1/n$
in the primal stability bound, is about $1.0\times10^{5}$, against an observed $n\,\varepsilon^w_T\approx7.3$, and the
statistical constant of Theorem~\ref{thm:pop_excess_risk} is about $3.0\times10^{12}$, against an
observed $n\cdot\mathrm{MSBE}\approx0.41$: with uniform constants the bounds are vacuous at
this scale. The worst values come from early iterates, where the dual variable is far from its maximizer.
Using the median constants after $t=5000$ ($\mu_{\mathrm{PL}}\approx0.95$,
$\rho\approx2.2$, $\mu_{\mathrm{QG}}\approx0.11$) gives $S_w\approx25$,
within a factor of about $3$ of the observed value, while the constant of
Theorem~\ref{thm:pop_excess_risk} (about $1.7\times10^{3}$) remains loose because of its
$(1+L/\rho)^2L_\Phi/(\mu_{\mathrm{PL}}\mu_{\mathrm{QG}})$ dependence. The experiment therefore supports
the $n^{-1}$ rate, not the size of the constants.

\begin{table}[t]
\centering
\footnotesize
\setlength{\tabcolsep}{4pt}
\caption{$n$ times the primal argument stability $\varepsilon^w_T$ (mean $\pm$ standard error over 50
repetitions, $T=40{,}000$). A constant column corresponds to an $n^{-1}$ rate. The last two rows give the
fitted slope of $\log\varepsilon^w_T$ against $\log n$ and its $95\%$ bootstrap interval.}
\label{tab:exp-stability}
\begin{tabular}{@{}rcccc@{}}
\toprule
 & \multicolumn{3}{c}{i.i.d.\ data} & single trajectory \\
\cmidrule(lr){2-4}
$n$ & $H=8$ & $H=16$ & $H=32$ & $H=16$ \\
\midrule
100 & $16.54$ {\scriptsize$\pm2.93$} & $8.95$ {\scriptsize$\pm1.38$} & $5.93$ {\scriptsize$\pm0.65$} & $9.93$ {\scriptsize$\pm1.20$} \\
200 & $11.88$ {\scriptsize$\pm1.33$} & $7.74$ {\scriptsize$\pm0.57$} & $7.18$ {\scriptsize$\pm0.61$} & $7.70$ {\scriptsize$\pm1.06$} \\
400 & $13.13$ {\scriptsize$\pm2.93$} & $8.38$ {\scriptsize$\pm0.78$} & $7.39$ {\scriptsize$\pm0.65$} & $6.77$ {\scriptsize$\pm0.55$} \\
800 & $8.45$ {\scriptsize$\pm0.56$} & $6.72$ {\scriptsize$\pm0.59$} & $6.27$ {\scriptsize$\pm0.44$} & $6.58$ {\scriptsize$\pm0.68$} \\
1600 & $8.84$ {\scriptsize$\pm0.73$} & $6.95$ {\scriptsize$\pm0.62$} & $5.59$ {\scriptsize$\pm0.55$} & $7.20$ {\scriptsize$\pm0.53$} \\
3200 & $10.32$ {\scriptsize$\pm1.14$} & $7.65$ {\scriptsize$\pm0.70$} & $5.75$ {\scriptsize$\pm0.48$} & $7.40$ {\scriptsize$\pm0.85$} \\
6400 & $9.03$ {\scriptsize$\pm0.77$} & $6.33$ {\scriptsize$\pm0.66$} & $6.05$ {\scriptsize$\pm0.54$} & $6.86$ {\scriptsize$\pm0.74$} \\
\midrule
slope ($n\ge 200$) & $-1.08$ & $-1.05$ & $-1.07$ & $-1.01$ \\
95\% interval & {\scriptsize$[-1.16,-1.00]$} & {\scriptsize$[-1.11,-0.99]$} & {\scriptsize$[-1.13,-1.01]$} & {\scriptsize$[-1.09,-0.93]$} \\
\bottomrule
\end{tabular}
\end{table}

\begin{table}[t]
\centering
\footnotesize
\setlength{\tabcolsep}{4pt}
\caption{$n$ times the excess population MSBE of the SGDA output (mean $\pm$ standard error), and of the
exact empirical minimizer on the datasets of the i.i.d.\ $H=16$ runs.}
\label{tab:exp-msbe}
\begin{tabular}{@{}rccccc@{}}
\toprule
 & \multicolumn{3}{c}{i.i.d.\ data} & single trajectory & exact empirical \\
\cmidrule(lr){2-4}
$n$ & $H=8$ & $H=16$ & $H=32$ & $H=16$ & minimizer (i.i.d.) \\
\midrule
100 & $0.480$ {\scriptsize$\pm0.035$} & $0.834$ {\scriptsize$\pm0.127$} & $0.508$ {\scriptsize$\pm0.036$} & $0.819$ {\scriptsize$\pm0.274$} & $0.672$ \\
200 & $0.457$ {\scriptsize$\pm0.031$} & $0.440$ {\scriptsize$\pm0.039$} & $0.448$ {\scriptsize$\pm0.025$} & $0.461$ {\scriptsize$\pm0.038$} & $0.436$ \\
400 & $0.447$ {\scriptsize$\pm0.037$} & $0.393$ {\scriptsize$\pm0.020$} & $0.446$ {\scriptsize$\pm0.025$} & $0.411$ {\scriptsize$\pm0.026$} & $0.391$ \\
800 & $0.420$ {\scriptsize$\pm0.026$} & $0.412$ {\scriptsize$\pm0.022$} & $0.391$ {\scriptsize$\pm0.021$} & $0.430$ {\scriptsize$\pm0.029$} & $0.404$ \\
1600 & $0.433$ {\scriptsize$\pm0.029$} & $0.385$ {\scriptsize$\pm0.020$} & $0.383$ {\scriptsize$\pm0.020$} & $0.375$ {\scriptsize$\pm0.022$} & $0.358$ \\
3200 & $0.464$ {\scriptsize$\pm0.029$} & $0.396$ {\scriptsize$\pm0.024$} & $0.414$ {\scriptsize$\pm0.025$} & $0.429$ {\scriptsize$\pm0.029$} & $0.361$ \\
6400 & $0.527$ {\scriptsize$\pm0.033$} & $0.443$ {\scriptsize$\pm0.026$} & $0.455$ {\scriptsize$\pm0.022$} & $0.486$ {\scriptsize$\pm0.023$} & $0.350$ \\
\midrule
slope ($n\ge 200$) & $-0.97$ & $-1.00$ & $-1.01$ & $-0.99$ & $-1.06$ \\
95\% interval & {\scriptsize$[-1.01,-0.92]$} & {\scriptsize$[-1.05,-0.95]$} & {\scriptsize$[-1.04,-0.97]$} & {\scriptsize$[-1.03,-0.94]$} & {\scriptsize$[-1.11,-1.02]$} \\
\bottomrule
\end{tabular}
\end{table}
\end{editblock}

\newpage
\section*{NeurIPS Paper Checklist}

\begin{enumerate}

\item {\bf Claims}
    \item[] Question: Do the main claims made in the abstract and introduction accurately reflect the paper's contributions and scope?
    \item[] Answer: \answerYes{}
    \item[] Justification: The abstract and introduction state the main claims: an $O(1/n)$ on-average primal argument-stability bound for SGDA (Theorem~\ref{thm:sgda_stability_noA7}) and $O(1/n+1/T)$ generalization and excess-risk bounds for soft BRM (Theorems~\ref{thm:empRM_gen} and~\ref{thm:pop_excess_risk}), under the primal and inner PL, QG and domain conditions of Section~\ref{sec:stability}. Section~\ref{sec:exp} illustrates the rates numerically.

\item {\bf Limitations}
    \item[] Question: Does the paper discuss the limitations of the work performed by the authors?
    \item[] Answer: \answerYes{}
    \item[] Justification: Section~5 (Limitations and future work) discusses the width requirement, the conditions that are hardest to verify, the size of the constants, and the restriction to smooth entropy-regularized BRM; Appendix~\ref{appssec:limitation} discusses non-smooth BRM and policy-performance guarantees.

\item {\bf Theory assumptions and proofs}
    \item[] Question: For each theoretical result, does the paper provide the full set of assumptions and a complete (and correct) proof?
    \item[] Answer: \answerYes{}
    \item[] Justification: The assumptions are Conditions~\ref{asm:smooth}--\ref{asm:unique} and Settings~\ref{asm:stepsize}--\ref{asm:uniform}. The results are Lemma~\ref{lem:sgda_global_conv_theta}, Theorem~\ref{thm:sgda_stability_noA7} (with a proof sketch in Section~\ref{sec:stability}), Lemma~\ref{lem:gen}, Theorems~\ref{thm:empRM_gen} and~\ref{thm:pop_excess_risk}, and Proposition~\ref{thm:sgda_stability_general}; complete proofs are in Appendix~\ref{appendix: Supporting Lemmas}--\ref{app:weakPD}.

    \item {\bf Experimental result reproducibility}
    \item[] Question: Does the paper fully disclose all the information needed to reproduce the main experimental results of the paper to the extent that it affects the main claims and/or conclusions of the paper (regardless of whether the code and data are provided or not)?
    \item[] Answer: \answerYes{}
    \item[] Justification: Section~\ref{sec:exp} and Appendix~\ref{app:experiments} specify the MDP, the two data regimes, the networks and their initialization, the SGDA hyperparameters, the number of repetitions, and how slopes and intervals are computed.

\item {\bf Open access to data and code}
    \item[] Question: Does the paper provide open access to the data and code, with sufficient instructions to faithfully reproduce the main experimental results, as described in supplemental material?
    \item[] Answer: \answerYes{}
    \item[] Justification: The code is available at \url{https://github.com/jajajang/soft-brm-stability}; the data are synthetic and generated by the released code (Appendix~\ref{app:experiments}).

\item {\bf Experimental setting/details}
    \item[] Question: Does the paper specify all the training and test details (e.g., data splits, hyperparameters, how they were chosen, type of optimizer) necessary to understand the results?
    \item[] Answer: \answerYes{}
    \item[] Justification: Appendix~\ref{app:experiments} lists all settings: MDP, data regimes, network widths and parameterization, minibatch size, stepsizes, dual multiplier, number of iterations, and sample sizes.

\item {\bf Experiment statistical significance}
    \item[] Question: Does the paper report error bars suitably and correctly defined or other appropriate information about the statistical significance of the experiments?
    \item[] Answer: \answerYes{}
    \item[] Justification: Tables~\ref{tab:exp-stability} and~\ref{tab:exp-msbe} report means with standard errors over 50 repetitions, and every fitted slope comes with a 95\% bootstrap interval over repetitions (Appendix~\ref{app:experiments}).

\item {\bf Experiments compute resources}
    \item[] Question: For each experiment, does the paper provide sufficient information on the computer resources (type of compute workers, memory, time of execution) needed to reproduce the experiments?
    \item[] Answer: \answerYes{}
    \item[] Justification: Appendix~\ref{app:experiments} states that all runs use NumPy on CPU and that the full study takes about one hour on two CPU cores.
    
\item {\bf Code of ethics}
    \item[] Question: Does the research conducted in the paper conform, in every respect, with the NeurIPS Code of Ethics \url{https://neurips.cc/public/EthicsGuidelines}?
    \item[] Answer: \answerYes{} 
    \item[] Justification: The research is theoretical and does not involve human subjects, sensitive data, deployed systems, or data/model releases with misuse risk. To the best of our knowledge, it conforms to the NeurIPS Code of Ethics.

\item {\bf Broader impacts}
    \item[] Question: Does the paper discuss both potential positive societal impacts and negative societal impacts of the work performed?
    \item[] Answer: \answerNA{} 
    \item[] Justification: The paper provides foundational theoretical guarantees for soft BRM and does not propose a deployed system, collect data, release a model, or target a specific application domain. Any societal impact would be indirect and mediated by future applications of offline reinforcement learning.

    
\item {\bf Safeguards}
    \item[] Question: Does the paper describe safeguards that have been put in place for responsible release of data or models that have a high risk for misuse (e.g., pre-trained language models, image generators, or scraped datasets)?
    \item[] Answer: \answerNA{} 
    \item[] Justification: The paper does not release data, models, or systems that pose a high risk of misuse. Therefore, safeguards for responsible release are not applicable.

\item {\bf Licenses for existing assets}
    \item[] Question: Are the creators or original owners of assets (e.g., code, data, models), used in the paper, properly credited and are the license and terms of use explicitly mentioned and properly respected?
    \item[] Answer: \answerNA{}
    \item[] Justification: The experiments use a synthetic MDP and standard open-source numerical libraries (NumPy, Matplotlib); no existing datasets, models or third-party code are used.

\item {\bf New assets}
    \item[] Question: Are new assets introduced in the paper well documented and is the documentation provided alongside the assets?
    \item[] Answer: \answerYes{}
    \item[] Justification: The experiment code is released at \url{https://github.com/jajajang/soft-brm-stability} with a README that describes the files and the commands that reproduce every result.

\item {\bf Crowdsourcing and research with human subjects}
    \item[] Question: For crowdsourcing experiments and research with human subjects, does the paper include the full text of instructions given to participants and screenshots, if applicable, as well as details about compensation (if any)? 
    \item[] Answer: \answerNA{} 
    \item[] Justification: The paper does not involve crowdsourcing, user studies, surveys, or any research with human subjects.

\item {\bf Institutional review board (IRB) approvals or equivalent for research with human subjects}
    \item[] Question: Does the paper describe potential risks incurred by study participants, whether such risks were disclosed to the subjects, and whether Institutional Review Board (IRB) approvals (or an equivalent approval/review based on the requirements of your country or institution) were obtained?
    \item[] Answer: \answerNA{} 
    \item[] Justification: The paper does not involve human subjects research. Therefore, IRB approval or equivalent review is not applicable.

\item {\bf Declaration of LLM usage}
    \item[] Question: Does the paper describe the usage of LLMs if it is an important, original, or non-standard component of the core methods in this research? Note that if the LLM is used only for writing, editing, or formatting purposes and does \emph{not} impact the core methodology, scientific rigor, or originality of the research, declaration is not required.
    \item[] Answer: \answerNA{} 
    \item[] Justification: The core method, theoretical results, and proofs do not use LLMs as an important, original, or non-standard component. Any use of LLMs, if any, was limited to writing, editing, or formatting assistance and did not affect the scientific methodology, rigor, or originality of the work.

\end{enumerate}

\end{document}